\documentclass[11pt,a4]{article}
\usepackage{fullpage}

\usepackage[utf8]{inputenc} 
\usepackage[T1]{fontenc}    
\usepackage{hyperref}       
\usepackage{url}            
\usepackage{booktabs}       
\usepackage{amsfonts}       
\usepackage{nicefrac}       
\usepackage{microtype}      
\usepackage[toc,page]{appendix}

\usepackage{graphicx}
\usepackage{caption}
\usepackage{subcaption}
\usepackage{amsmath,amssymb}
\usepackage{mathtools}
\usepackage{epstopdf}
\usepackage[makeroom]{cancel}
\usepackage{breqn}
\usepackage[english]{babel}
\usepackage[utf8]{inputenc}
\usepackage{algorithm}
\usepackage[noend]{algpseudocode}

\usepackage{xcolor}

\renewcommand{\b}{\boldsymbol}
\renewcommand{\a}{\alpha}
\newcommand{\R}{\mathbb{R}}

\newcommand{\eps}{\epsilon}

\usepackage{dsfont,wasysym,euscript,mathrsfs,amsthm,bm}
\def\sP{\mathscr{P}}
\DeclareMathAlphabet{\pazocal}{OMS}{zplm}{m}{n}
\def\sF{\pazocal{F}}
\def\bnu{\boldsymbol{\nu}}
\def\blambda{\boldsymbol{\lambda}}
\def\brho{\boldsymbol{\rho}}
\def\bsigma{\boldsymbol{\sigma}}
\def\barmu{\bar{\mu}}
\def\rn{\mathbb{R}}
\def\sD{\pazocal{D}}
\def\sQ{\mathscr{Q}}

\def\bx{\mathbf{x}}
\def\bX{\mathbf{X}}

\def\bb{\mathbf{b}}
\def\bq{\mathbf{q}}

\def\bA{\mathbf{A}}
\def\bB{\mathbf{B}}
\def\b1{\mathbf{1}}
\def\nn{\mathbb{N}}
\def\spn{\operatorname{span}}
\def\proj{\operatorname{proj}}
\def\tbb{\widetilde{\mathbf{b}}}
\def\tS{\widetilde{S}}
\def\simiid{\overset{iid}{\sim}}

\newcommand\numberthis{\addtocounter{equation}{1}\tag{\theequation}}
\usepackage{amsmath}

\DeclareMathOperator*{\argmin}{arg\,min}

\newtheorem{theorem}{Theorem}
\newtheorem*{theorem*}{Theorem}
\newtheorem{definition}{Definition}
\newtheorem{lemma}{Lemma}
\newtheorem{corollary}{Corollary}

\newtheorem{manualtheoreminner}{Theorem}
\newenvironment{manualtheorem}[1]{%
  \renewcommand\themanualtheoreminner{#1}%
  \manualtheoreminner
}{\endmanualtheoreminner}

\title{Consistent Estimation of Identifiable Nonparametric Mixture Models from Grouped Observations}

\author{%
  Alexander Ritchie\thanks{Equal contribution.} \\
  Department of EECS\\
  University of Michigan\\
  Ann Arbor, MI 48109 \\
  \texttt{aritch@umich.edu} \\
   \and
   Robert A. Vandermeulen\footnotemark[1] \\
   ML group \\
    Technische Universität Berlin \\
    10587 Berlin, Germany \\
    \texttt{vandermeulen@tu-berlin.de}
   \and
   Clayton Scott \\
  Departments of EECS, Statistics\\
  University of Michigan\\
  Ann Arbor, MI 48109 \\
  \texttt{clayscot@umich.edu}
}

\begin{document}
\date{}
\maketitle

\begin{abstract}
Recent research has established sufficient conditions for finite mixture models to be identifiable from grouped observations. These conditions allow the mixture components to be nonparametric and have substantial (or even total) overlap. This work proposes an algorithm that consistently estimates any identifiable mixture model from grouped observations. Our analysis leverages an oracle inequality for weighted kernel density estimators of the distribution on groups, together with a general result showing that consistent estimation of the distribution on groups implies consistent estimation of mixture components. A practical implementation is provided for paired observations, and the approach is shown to outperform existing methods, especially when mixture components overlap significantly. 
\end{abstract}

\section{Introduction}
In statistics and machine learning, finite mixture models are often used to describe the distribution of subpopulations within a larger population. A finite mixture model can be written
\begin{equation} \label{eq:mm}
p=\sum_{m=1}^{M}w^*_mp^*_m,
\end{equation}
where $w^*_m >0$ are mixing weights such that $\sum_{m=1}^M w^*_m = 1$, and $p^*_m$ are probability densities. Without additional assumptions, the mixture model $p$ is not identifiable from iid data. Typically, identifiability is ensured by restricting the $p^*_m$ to some family of parametric distributions. Restricting the $p^*_m$ to be Gaussian yields the Gaussian mixture model (GMM) which is identifiable \cite{anderson2014more, yakowitz1968identifiability}.

Most work on estimating mixture models assumes an iid sampling scheme. In this work we examine an alternative sampling scheme where observations occur in iid groups. Each group is generated by sampling a component $m\in[M]$ according to $w_m^*$, and then drawing $N$ iid observations from $p_m^*$.

Recent work has shown that \textit{any} finite mixture model is identifiable given grouped observations of sufficient size \cite{vandermeulen_operator_2019}. In the worst case, any finite mixture model with $M$ components is identifiable given groups of size $N \ge 2M-1$. It was also shown that, if the underlying components of the mixture model are \textit{jointly irreducible} \cite{blanchard2014decontamination}, then the mixture is identifiable given paired observations ($N=2$). This framework provides a setting where the potential exists to recover nonparametric and highly overlapping mixture components. As of yet, however, no general theory or algorithms are known for this estimation problem.

This work makes the following contributions. We introduce a novel variant of the kernel density estimator that yields statistically consistent estimates of \emph{any identifiable} nonparametric mixture model (NoMM) from grouped observations. To prove this result, we establish an oracle inequality for weighted kernel density estimators. We also establish a general result showing that consistent estimation (with an estimator possessing a natural factored form) of the distribution on groups implies consistent estimates of the underlying components when the NoMM is identifiable. The only additional condition imposed by our theory is that the $p_m^*$ be square integrable. In the case of $N=2$, we offer an efficient algorithm and demonstrate its effectiveness on several datasets. 

We study two applications where paired observations naturally arise. The first is nuclear source detection, where nuclear particles interact with a detector to produce some form of measurement. A critical challenge in this application is to classify incoming particles as belonging to source or background. Because of changing environments, training data are typically not available, and these two classes also have substantial overlap. By positioning two detectors side-by-side, it is possible to simultaneously measure two particles from the same (unknown) class. 

We also apply our method to topic modeling of Twitter data. Since tweets usually express a small set of very closely related ideas, words in tweets contain common underlying semantic information. The pairing of words has the potential to encode this semantic information in a way that accounts for context. The proposed method, which operates on continuous word embeddings, allows for flexible modeling of the distributions of topics over words using static word embeddings \cite{pennington2014glove}. Furthermore our method does not require anchor words, allows for substantial overlap of topics without loss of identifiability, and can be trained using documents with as few as two words \emph{without} any document aggregation \cite{gao2019incorporating, qiang2017topic}. While other works have explored topic modeling with word embeddings, which we call continuous topic modeling, most either impose parametric assumptions or are not suited for very short texts. To our knowledge, this is the first work to consider a nonparametric approach to continuous topic modeling of very short texts.

\section{Background and Previous Work}

Much of the literature concerning NoMMs falls in the category of Bayesian nonparametrics, a thorough summary of which is given in \cite{xuan2019survey}. Typically, mixture models in this setting do not assume that the number of mixture components is known, and instead assume that the mixture components are from a known parametric family of distributions. An in-depth treatment of Bayesian NoMMs (BNoMMs) can be found in \cite{lindsay1995mixture}. The parametric assumptions on the mixture components have been relaxed in \cite{aragam2018identifiability}, but the identifiability results impose regularity and separation conditions on the components. We mention BNoMMs only for completeness and emphasize that our work considers an alternative setting where the number of mixture components is known, but few to no assumptions are made on the mixture components themselves.

Mixture models are often utilized to solve the clustering problem. Parametric mixture models, such as GMMs, are able to capture overlapping clusters. 
Most clustering algorithms, however, such as $k$-means \cite{forgy1965cluster, lloyd1982least}, DBSCAN \cite{ester1996density}, and spectral clustering \cite{malik00spectral, luxburg07tutorial}, assume clusters are non-overlapping and hence fail when clusters overlap. The grouped observation setting considered in this work is known in the clustering literature as \emph{clustering with instance-level constraints} \cite{wagstaff2000clustering, wagstaff2001constrained, wagstaff2006constrained}. A survey of constrained clustering is given in \cite{ganccarski2020constrained}. Grouped observations correspond to so-called must-link constraints, where two or more observations are known, through expert knowledge or some other means, to belong to the same cluster. Most constrained clustering approaches cannot model overlapping clusters effectively \cite{scripps2010constrained}.


There is relatively little work on mixture modeling with nonparametric components, and to our knowledge no prior work addresses the incorporation of instance-level constraints in the NoMM setting. Mallapragada et al. \cite{mallapragada2010non} use a mixture of kernel density estimators to estimate a NoMM, but do not address identifiability or provide statistical guarantees. Aragam et al. \cite{aragam2018identifiability} prove identifiability of NoMMs under regularity and separation conditions on the components. They provide a simple algorithm that gives Bayes optimal cluster assignments in the limit, but they does not guarantee on recovery of the underlying components. Schiebinger et al. \cite{schiebinger2015geometry} study kernelized spectral clustering and characterize recoverability of components with small overlap. Zheng and Wu \cite{zheng2019nonparametric} establish consistent estimation of NoMMs under the assumption that mixture components have independent marginals. Bao et al. \cite{bao2018classification} consider the related problem of ``similar-unsupervised'' binary classification, which assumes access to unlabeled data in addition to must-link constraints.

In the grouped observation setting, previous works on multi-view models can be adapted to prove identifiability results and give algorithms to recover mixture model components. When the mixture components are linearly independent it has been shown that three observations per group is sufficient to yield identifiability as well as an algorithm to provably recover the components \cite{anandkumar14,allman09}. We note that these approaches require three observations per group, while the proposed method works with as few as two observations per group. This difference amounts to performing kernel density estimation in three times the ambient dimension versus two. With the instability of KDEs in high dimension, the reduction to groups of size two can be very meaningful in practice. Furthermore, in applications like nuclear particle classification, triples may be exceedingly rare or difficult to measure. For discrete data, similar results from nonnegative matrix factorization exist under joint irreducibility with two observations per group \cite{arora12}, and algorithms have been proposed to recover arbitrary mixture models with $M$ components given $2M-1$ observations per group \cite{rabani14,vandermeulen_operator_2019}.

\section{Notation}
For $1 \le p < \infty$ denote $\left \lVert f\right \rVert_p :=(\int_{\R^d} |f(x)|^p dx)^{1/p}$, and $L^p:= \{f:\R^d \to \R : \| f \|_p < \infty\}$.
The transpose of a matrix $A$ will be written $A'$. Random variables will be referred to by capital letters, and instances of random variables will be referred to by the corresponding lowercase letter. 
We represent the set of positive integers $\{1,2,\dots,M\}$ by $[M]$. We let $\Delta^R$ be the probability simplex in $\mathbb{R}^R$. We denote the $M$-fold Cartesian product of a set with a subscript, e.g., $\Delta^R_M = \underbrace{\Delta^R \times \dots \times \Delta^R}_{M}$.

\section{Problem Statement}
We precisely introduce the grouped observation setting, review known identifiability results, and formalize the estimation problem. The standard sampling procedure for a mixture model of the form $p=\sum_{m=1}^{M}w^*_mp^*_m$ can be viewed as a two step process wherein one samples a mixture component $p_m^*$ with probability $w_m^*$ and then observes one draw from that distribution $X\sim p_m^*$. The \emph{grouped observation setting} considers an alternative sampling scheme where, after selecting a mixture component $p_m^*$, instead of only drawing a single observation, a group of observations $\bX = (X_1, \ldots, X_N)$ are drawn iid from $p_m^*$. As in a standard mixture model, one does not know \emph{a priori} from which mixture component a grouped observation is sampled. Repeating this $n$ times, one's data consists of $n$ groups of $N$ observations per group $\bX_1 = \left(X_{1,1},\ldots,X_{1,N} \right), \ldots, \bX_n = \left(X_{n,1},\ldots,X_{n,N} \right)$. The distribution on groups is $\bX \stackrel{iid}{\sim} 
\sum_{m=1}^M w_m^* {p_m^*}^{\times N}$,
where ${p_m^*}^{\times N}: \R^d_N \to \R$ denotes the product density such that ${p_m^*}^{\times N}(y_1, y_2, \dots, y_N) = p_m^*(y_1)p_m^*(y_2)\dots p_m^*(y_N)$. Note that when $N=1$ this is simply a standard mixture model.

Vandermeulen and Scott \cite{vandermeulen_operator_2019} characterized identifiability from grouped observations for mixtures of general probability measures. A mixture model $p = \sum_{m=1}^M w_m^* p_m^*$ is said to be \emph{$N$-identifiable} if $p$ cannot be expressed 
$p=\sum_{m=1}^{M'} w_m' p_m'$ for some distinct mixture model such $M' \le M$ and $\sum_{m=1}^M w_m^* {p_m^*}^{\times N} = \sum_{m=1}^{M'} w_m' {p_m'}^{\times N}$. In words, $N$-identifiability of $p$ means there is no other mixture model with $M$ or fewer components that induces the same distribution on groups. They show that a general mixture model is $N$-identifiable from grouped observations provided $N \ge 2M-1$, and that this cannot be improved without imposing restrictions on the components. The result places no assumptions whatsoever on the components.

In practice, the bound of $2M-1$ is probably pessimistic, and the most useful cases are likely when $N$ is small, say two or three. The authors of \cite{vandermeulen_operator_2019} also show that if the $p_m^*$ are \emph{jointly irreducible} (\emph{linearly independent}), then the mixture is $N$-identifiable for $N=2$ ($N=3$). A collection of probability densities $\mu_1,\mu_2,\dots,\mu_M$ is said to be \emph{jointly irreducible} (JI) if $\sum_{m=1}^M c_m\mu_m$ is never a valid density whenever some $c_m < 0$. JI is satisfied, for example, if the support of each mixture component has some subset of positive measure that does not intersect the supports of the other mixture components (a continuous analogue of the anchor word assumption). This is not necessary, however; JI is still possible if all densities have the same support. In the remainder of the paper we focus on the setting of $N=2$, not only because JI provides a flexible nonparametric condition where paired observations suffice, but also because the notation for our estimator becomes cumbersome when $N > 2$. Our theory generalizes easily to $N > 2$, and these details are described in Appendices \ref{a:t1}, \ref{a:t2}, and \ref{a:t3}.

The paired observations $\bX_1,\dots,\bX_n$ with $\bX_i=(X_{i,1},X_{i,2}) \in \mathbb{R}^d \times \mathbb{R}^d$ are iid and have density 
\begin{equation}
q(x,x') \coloneqq \sum_{m=1}^M w^*_m p^*_m(x) p^*_m(x') \quad x,x' \in \R^d.\label{q}
\end{equation}
We assume $M$ is known. Our goal is to consistently estimate $w_m^*$ and $p_m^*$  when $p$ is identifiable. 

\section{A Weighted Kernel Density Estimator}

Our overall strategy is to first devise a consistent estimator of $q$, the density on pairs, where the estimator has a factorized form reflecting the group sampling scheme. In the next section we prove that if an estimator for $q$ is consistent, and $p$ is identifiable, then the components comprising our estimator converge to the true components.

Let $k: \R^d \to \R$ be a function, called a \emph{kernel}, such that $k \ge 0$ and $\int k(x) dx = 1$. An example is the Gaussian kernel $k(x) = (2 \pi)^{-1/2} \exp(-\|x\|^2/2)$. For $\sigma > 0$, define $k_\sigma(x,x'):= \sigma^{-d} k((x - x')/\sigma)$. We refer to the second argument of $k_\sigma$ as the \emph{center} of the kernel. A \emph{weighted kernel density estimator} (wKDE) for a density on $\R^d$, but constructed from the paired observations $\bX_i$, has the form
\begin{equation*}
    p(x;\theta) = \sum_{r=1}^n \sum_{r'=1}^2 \theta_{r,r'} k_\sigma(x,X_{r,r'}), \label{p}
\end{equation*}
where $\theta_{r,r'}$ is the element of $\theta = [\theta_{1,1}, \theta_{1,2}, \dots, \theta_{n,1}, \theta_{n,2}]'$ corresponding to the weight of the kernel centered at $X_{r,r'}$. We propose to model the mixture components as wKDEs. Specifically, given $n$ paired observations, we consider estimators of $q$ of the form
\begin{equation}
   q_{w,\a}(x,x') = \sum_{m=1}^M w_m p(x;\a_m) p(x';\a_m),  \label{qhat}
\end{equation}
where $w = [w_1, w_2, \dots, w_M]' \in \Delta^{M}$, $\a_m = [\a_{m,1,1}, \a_{m,1,2}, \dots, \a_{m,n,1}, \a_{m,n,2}]' \in \Delta^{2n}$ for all $m \in [M]$, with $\a_{m,r,r'}$ corresponding to the weight of the kernel centered at $X_{r,r'}$ in the estimate of the $m^{th}$ mixture component, and $\a \coloneqq (\a_1, \ \a_2, \dots, \ \a_M) \in \Delta_M^{2n}$.

To select the parameters $(w,\a)$, we propose to minimize the integrated square error (ISE) of $q_{w,\a}$ given by$\left \lVert q - q_{w,a}\right \rVert_2^2 \coloneqq \int [ q(x,x')-q_{w,a}(x,x')]^2 dx  dx'$.
Expanding the ISE gives
\begin{align*}
  \left \lVert q - q_{w,a}\right \rVert_2^2 = &\int q_{w,\a}^2(x,x') dx dx' - 2 \int q_{w,\a} (x,x')q(x,x')dx dx' + \cancelto{\text{const.}}{\int q^2(x,x')dx dx'}.
\end{align*}
Since the final term is constant with respect to $w$ and $\a$, we focus on minimizing the first two terms which we call the truncated ISE (TISE) and denote by $J(w,\a)$. Substituting the definition of $q_{w,\a}$ in the TISE yields
\begin{align}
     J(w,\a) \coloneqq \int q_{w,\a}^2(x,x') dx dx' 
    -2\sum_{m=1}^M \sum_{r=1}^n \sum_{r'=1}^2 \sum_{s=1}^n \sum_{s'=1}^2 w_m  \a_{m,r,r'} \a_{m,s,s'} h(r,r',s,s'), \label{tise}
\end{align}
where $h(r,r',s,s') \coloneqq \int k_\sigma(x,X_{r,r'})k_\sigma(x',X_{s,s'})q(x,x')dxdx'$. Since $q$ is unknown, the ISE and therefore $J(w,a)$ cannot be calculated directly. Noting that $h(r,r',s,s')$ is an expectation, we estimate this term using a hybrid leave-one-out/leave-two-out (LOO/LTO) estimator
\begin{align*}
    \hat h(r,r',s,s') &\coloneqq \begin{cases}
     \frac{1}{n-2} \sum_{i \in [n] \setminus \{r,s\}}k_\sigma(X_{i,1},X_{r,r'})k_\sigma(X_{i,2},X_{s,s'}) & r \neq s \\
    \frac{1}{n-1} \sum_{i \in [n] \setminus \{r\}}k_\sigma(X_{i,1},X_{r,r'})k_\sigma(X_{i,2},X_{s,s'}) & r=s
    \end{cases}.
\end{align*}

In this manner we have the empirical TISE
\begin{align}
    \hat J(w,\a) &\coloneqq \int q_{w,\a}^2(x,x') dx dx' 
    -2\sum_{m=1}^M \sum_{r=1}^n \sum_{r'=1}^2 \sum_{s=1}^n \sum_{s'=1}^2 w_m  \a_{m,r,r'} \a_{m,s,s'} \hat h(r,r',s,s').
\end{align}
With all the notation in place, our estimate of the nonparametric mixture model is determined by
\begin{equation}
    (\hat w, \hat \a) \coloneqq \ \argmin_{w \in \Delta^M, \ \a \in \Delta_M^{2n}} \ \hat J(w,\a), \label{eq:generic_problem}
\end{equation}
where $\hat w_m$ are the mixing weights and $p(x;\hat\a_m)$ are the mixture components for $m \in \left[M\right]$. The theoretical results presented in Section \ref{sec:theory} concern the behavior of the minimizer of \eqref{eq:generic_problem}. We show not only that the empirical TISE minimizing estimator $\hat q \coloneqq q_{\hat w, \hat \a}$ consistently estimates $q$, but its components also consistently estimate the underlying mixture model if it is identifiable.

\section{Theoretical Results} \label{sec:theory}
In this section we state our assumptions and main results. Formal proofs are given in the Appendices \ref{a:t1}, \ref{a:t2}, and \ref{a:t3}. Our overall approach is to first show that the proposed $\hat{q}$ is a consistent estimate of $q$ (Theorems \ref{oe} and \ref{l3}). We then show that if $p$ in \eqref{eq:mm} is identifiable, then the components $p(x;\hat{\a}_m)$ defining $\hat q$ are consistent estimates of $p_m^*$, as are the $\hat{w}_m$ for $w_m^*$ (Theorem \ref{thm:compconv-maintext}).

We assume throughout this section that $p_m^* \in L^2$ for all $m$. We also require that the kernel $k$ satisfy two additional conditions: $k \in L^2$ and $k \le C_k$ for some constant $C_k < \infty$.

We begin with an oracle inequality, which shows that our estimator selects an approximately optimal member of our model class.
\begin{theorem} \label{oe} 
Let $\epsilon>0$ and set $\delta = 8(n^2-n)\exp\{-\frac{\sigma^{4d}(n-2)\epsilon^2}{8C_k^4}\} + 8n\exp\{-\frac{\sigma^{4d}(n-1)\epsilon^2}{8C_k^4}\}$. With probability at least $1-\delta$ the following holds: $\left \lVert q - q_{\hat w, \hat \a} \right \rVert_2^2 \leq \  \inf_{w \in \Delta^M, \ \a \in \Delta_M^{2n}} \left \lVert q - q_{w, \a} \right \rVert_2^2 + \epsilon.$
\end{theorem}
\begin{proof}[Proof Sketch] The estimators $\hat{h}$ are constructed so that they are sums of independent random variables, allowing us to apply Hoeffding's inequality to show that each $\hat{h}$ concentrates around its $h$. Then using basic inequalities (triangle inequality, union bound) and the simplex constraints on $w$ and $\a$, we show that $\hat{J}(w,\a)$ concentrates around $J(w,\a)$ uniformly over the parameter space.
\end{proof}
The next result uses Theorem \ref{oe} to establish that $\hat{q}$ is a consistent estimate of $q$ in the $L^1$ norm.
\begin{theorem}\label{l3}
If $\sigma \to 0$ and $\frac{n\sigma^{4d}}{\log n} \to \infty$ as $n \to \infty$, then $\left \lVert q - q_{\hat w, \hat \a}\right \rVert_1 \xrightarrow{a.s.} 0$.
\end{theorem}
\begin{proof}[Proof Sketch] We appeal to a result of \cite{gyorfi90} showing that if $\int \hat{q} = 1$, which it does in our case, then strong consistency (i.e., a.s. convergence) of a density estimator in $L^2$ implies strong consistency in $L^1$. To show strong consistency in $L^2$, from Theorem \ref{oe} it suffices to exhibit $w \in \Delta^M$ and $\a \in \Delta^{2n}_M$ such that $\left \lVert q - q_{w, \a}\right \rVert_1 \xrightarrow{a.s.} 0.$ For this we take $w = w^*$ and $\a = \a^*$ such that each $\a_m^*$ is uniform on the data points drawn from $p_m^*$. This makes $p(\cdot \, ; \a_m^*)$ the usual (uniformly weighted) KDE for $p_m^*$, which is known to be a strongly consistent estimator. Strong consistency of $\hat{q}$ then easily follows.
\end{proof}
The preceding results hold regardless of whether $p$ in \eqref{eq:mm} is identifiable. The next result states that if $p$ is identifiable, then the estimates $p(\cdot \, ; \hat \a_m)$ comprising $\hat q$ are consistent estimates of the true components $p_m^*$, as are the $\hat w_m$ of $w_m^*$. The result is stated for $N \ge 2$.

\begin{theorem}
\label{thm:compconv-maintext}
  Let $\sum_{m=1}^M w_m p_m$ be an $N$-identifiable mixture model, and $\sum_{m=1}^M \hat{w}_{m,j} \hat{p}_{m,j}$ be a sequence of mixture models such that $\left\|\sum_{m=1}^M \hat{w}_{m,j} \hat{p}_{m,j}^{\times N} - \sum_{m=1}^M w_m p_m^{\times N} \right\|_1 \to 0$. Then there is a sequence of permutations $\sigma_j$ so that $\hat{w}_{\sigma_j(m),j} \to w_m$ and $\left\|\hat{p}_{\sigma_j(m),j} - p_m\right\|_1 \to 0$ for all $m$.
\end{theorem}
\begin{proof}[Proof Sketch]
    We show that if $\left\|\sum_{m=1}^M \hat{w}_{m,j} \hat{p}_{m,j}^{\times N} - \sum_{m=1}^M w_m p_m^{\times N} \right\|_1 \to 0$
    then the components $\hat{p}_{m,j}$ admit some convergent subsequence, and therefore so do $\hat{p}_{m,j}^{\times N}$. If a subsequence $\hat{p}_{m,j}^{\times N}$ stays away from the components $p_m^{\times N}$ then some subsequence would converge to a component other than some $p_m^{\times N}$. This allows us to construct a mixture model violating $N$-identifiability, a contradiction.
\end{proof}
This result has been stated in terms of densities for readability, but Appendix \ref{a:t3} contains a general measure-theoretic version. We may combine Theorems \ref{l3} and \ref{thm:compconv-maintext} to establish the following (returning to the setting of $N=2$). To our knowledge, this is the first result to establish consistent estimation, under any sampling scheme, of NoMMs with substantial overlap. 
\begin{corollary}\label{cor:main}
If $\sigma \to 0$ and $\frac{n\sigma^{4d}}{\log n} \to \infty$ as $n \to \infty$, and 
$p$ is $2$-identifiable (e.g., the $p_m^*$ are jointly irreducible), then $\hat w_m \stackrel{a.s.}{\to} w_m^*$ and $\|p(\cdot; \hat \alpha_m) - p_m^*\|_1 \stackrel{a.s.}{\to} 0$, up to a permutation.
\end{corollary}
The significance of the result is that joint irreducibility is both a flexible nonparametric assumption, while ensuring identifiability in the case $N=2$ for which a practical implementation of $\hat q$ is possible. We include an analogous result for \emph{all identifiable NoMMs} in Appendix \ref{a:c1}.

\section{Optimization}
In this section we suggest an approach for solving \eqref{eq:generic_problem}. We first consider the problem as presented up to this point, which we call the \emph{full problem}. We then consider an approach for speeding up optimization by heuristically choosing a coreset as the kernel centers, which we call the \emph{coreset approach}. 
In what follows, we assume that $\widetilde k_\sigma(z_r,z_u) \coloneqq \int k_\sigma(x,z_r) k_\sigma(x,z_u) dx$ has a closed-form expression or can otherwise be computed efficiently. This assumption is satisfied by many common kernels such as the Gaussian, Cauchy, and Laplacian kernels.

{\bf Form of the Optimization Problem. }
The optimization problem \eqref{eq:generic_problem} can be written
\begin{align}
    \min_{w \in \Delta^M, \ \a \in \Delta_M^R} \ \sum_{k=1}^M\sum_{\ell=1}^M w_k w_\ell \Bigg( \a_k' G \a_\ell \Bigg)^2 - 2 \sum_{m=1}^M w_m \left(\a_m'C\a_m\right), \label{optprob}
\end{align}
where the matrices $G,C\in \R^{R\times R}$ will be defined shortly. Details are given in Appendix \ref{a:opt}. In particular, both the full problem and the coreset approach can be written in the form of \eqref{optprob}, differing only in the definitions of $R$ and $G,C$. We therefore propose to use the same optimization approach for both problems. For the full problem, $R=2n$ and the matrices $G$ and $C$ have the form
\begin{equation*}
    G_{a,b} = \widetilde k_\sigma(X_{\lfloor\frac{a}{2}\rfloor, a\text{ mod } 2}, X_{\lfloor\frac{b}{2}\rfloor, b\text{ mod } 2}) \quad
    C_{a,b} = \hat h(\lfloor\frac{a}{2}\rfloor, a\text{ mod } 2, \lfloor\frac{b}{2}\rfloor, b\text{ mod } 2).
\end{equation*}
Though the problem \eqref{optprob} is nonconvex, we observe that a properly initialized alternating projected stochastic gradient descent (APSGD) procedure produces good solutions in practice.

Pseudocode for the APSGD algorithm for solving \eqref{optprob} is given in Appendix \ref{a:opt}. We mention that the projections are onto the probability simplex, a decaying step size is used, and stochasticity is introduced via the matrix $C^{(t)}$, which is a mini-batch version of $C$ defined by $C^{(t)}_{a,b} = \frac{1}{|\Omega^{(t)} \backslash \{a,b\}|}\sum_{i \in |\Omega^{(t)} \backslash \{a,b\}|} k_\sigma(X_{i,1},X_{\lfloor\frac{a}{2}\rfloor, a\text{ mod } 2})k_\sigma(X_{i,2},X_{\lfloor\frac{b}{2}\rfloor, b\text{ mod } 2})$,
where $\Omega^{(t)}$ is the index set corresponding to the $t^{th}$ mini-batch.

{\bf Coreset Approach.} 
KDEs traditionally center kernels at the location of each observation, i.e., $k_\sigma(\cdot, x_{i,i'})$, where $x_{i,i'}$ is the kernel center. Rather than constraining the wKDE to have kernels centered at the observations, we can formulate the optimization problem with $R$ kernel centers $z_r \in \R^d$ for some suitably chosen $z_r$, which we take to be our coreset. Further details are given in Appendix \ref{a:opt}. We note the per-batch computational complexity for our APSGD algorithm is dominated by the gradient calculations and calculating $C^{(t)}$. If we assume $R>M$, the total complexity is $\mathcal{O}(n_e n (M+d)R^2)$ where $n_e$ is the number of training epochs. Thus, choosing $R \ll 2n$ offers a substantial speed-up.

{\bf Initialization.} \label{sec:init}
We adopt a spectral initialization scheme. We focus on the full problem for concision, but the coreset approach is similar; further details for both are provided in Appendix \ref{a:opt}. By Lemmas 5.1 and 8.2 of Vandermeulen and Scott \cite{vandermeulen_operator_2019}, one can view the standard KDE on the full sample as a symmetric linear operator $T:L^2(\R^d) \to L^2(\R^d)$. We use the eigenvectors of $T$, which are wKDEs on $\R^d$, to form a low-rank approximation of the standard KDE initialize our algorithm. This initialization is a low-rank approximation of the standard KDE.

\section{Experiments}
In this section we compare our coreset approach against several competing methods on a number of real and highly overlapping synthetic datasets. Datasets are described in Table \ref{tab: datasets}. We call the proposed method Nonparametric Density estimation of Identifiable mixture models from Grouped Observations (NDIGO). All code and synthetic datasets are publicly available.\footnote{Available Online: \url{https://github.com/aritchie9590/NDIGO}} The MAGIC gamma ray detection dataset \cite{bock2004methods} is publicly available via the UCI machine learning repository. The Russian-troll-tweets Twitter dataset is publicly available through FiveThirtyEight.\footnote{Available Online: \url{https://github.com/fivethirtyeight/russian-troll-tweets}} For NDIGO and MVLVM, we used a Gaussian kernel in all experiments and Scott's rule \cite{scott2015multivariate} was used for bandwidth selection. For synthetic experiments, $R$ was selected to yield the initialization with the lowest empirical TISE. $R$ was chosen from $\{10,20,30,40,50\}$ for both moons datasets, and from $\{60,70,80,90,100\}$ for the Olympic rings and half-disks datasets. We used $R=200$ for the MAGIC and Twitter datasets.

Several of the methods we compare against do not produce density estimates, so we evaluate the clustering induced by each method. For constrained clustering methods, we compare against constrained spectral clustering (CSC) \cite{wang2014constrained}, and constrained GMM (CGMM) \cite{basu2008constrained}. We also compare against the NoMM methods NPMIX of Aragam et al.  \cite{aragam2018identifiability} and MVLVM of Song et al. \cite{song2014nonparametric}. MVLVM is our most similar competitor as it considers groups of size three. Each constrained clustering algorithm was given access to all pair information. MVLVM was supplied triplets from the training data. NPMIX does not utilize the pair information in any way. Following the literature, we report the clustering results for the training sample. Out-of-sample results are provided in Appendix \ref{a:exp}, but we mention NPMIX is the best performer. Parameters for CSC and NPMIX were optimized w.r.t. a separately generated holdout dataset. Average results over ten runs on the synthetic datasets are shown in Figure \ref{fig:clustering_assignments}. NDIGO outperforms all methods considered.
\begin{figure}[htb]
      \centering
      \includegraphics[width=\textwidth]{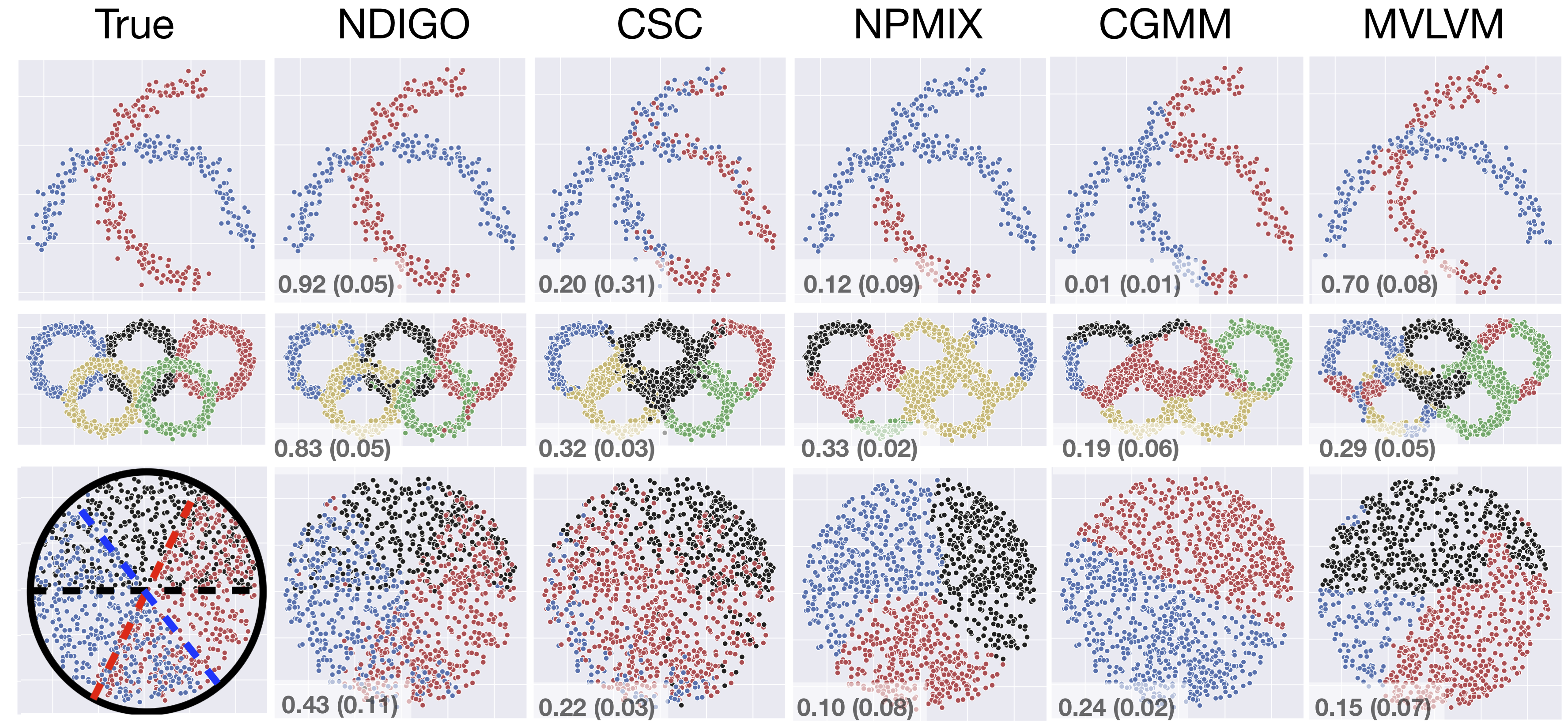}
      \caption{Example cluster assignments of four synthetic datasets by each method. Mean ARI (standard deviation) over $10$ runs is shown at the bottom left of each clustering (larger is better).The datasets are overlapping moons (top), Olympic rings (middle), and half-disks (bottom). Half-disks has been annotated to show the true components.}
    \label{fig:clustering_assignments}
\end{figure}
The synthetic datasets were constructed to have clusters that are non-ellipsoidal in shape with substantial overlap between clusters. The clusterings induced by each method are shown in Figure \ref{fig:clustering_assignments}. Performance is measured in terms of the adjusted Rand index (ARI) \cite{hubert1985comparing}. We observe that NDIGO gives superior performance across all experiments, especially when clusters have substantial overlap. Density estimates produced by our method for synthetic datasets are shown in Figure \ref{fig:kde}.

\begin{figure}[htb]
\centering
\begin{minipage}{0.62\textwidth}
    \includegraphics[width=\textwidth]{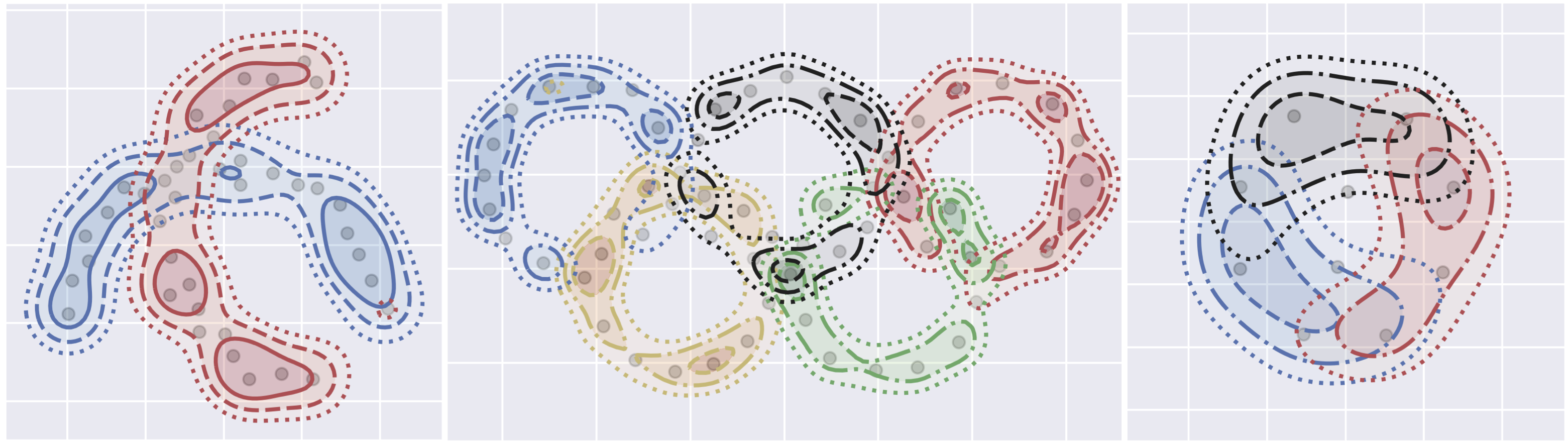}
    \caption{Component density estimate contours produced NDIGO. From left to right: overlapping moons, Olympic rings, half-disks}
    \label{fig:kde}
\end{minipage}
\begin{minipage}{0.37\textwidth}
\centering
\captionsetup{type=table} 
\hspace{2pt}
    \caption{Description of datasets. $^*$Quantities after preprocessing.}\label{tab: datasets}
  \begin{tabular}{llll}
    \toprule
    Dataset (2n) & $M$/$d$ \\
    \midrule
    Ovlp. Moons $(400)$ & $2$/$2$     \\
    Olympic Rings $(2000)$ & $5$/$2$  \\
    Half-disks $(1200)$ & $3$/$2$ \\
    MAGIC $(19,020)$ & $2$/$10$  \\
    Twitter $(3,382,162^*)$ & -/$10^*$ \\
    \bottomrule
  \end{tabular}
\end{minipage}
\end{figure}

Results on the MAGIC dataset are shown in Figure \ref{fig:roc}. The task is to detect gamma radiation events among background radiation. When detecting rare events, the proper performance indicator is given by the receiver operating characteristic (ROC) curve, which plots the true positive rate vs. the false positive rate, parameterized by the threshold of a likelihood ratio test (LRT). Each method was trained using $80\%$ of the available data, and the ROC curve was generated from the remaining $20\%$. CSC was excluded from this test because it does not produce a density estimate, so a LRT cannot be applied. As an upper bound on possible unsupervised performance, we trained KDEs on each class and plugged the resulting density estimates into an LRT. Previous studies concluded this method, which we call KDE-plugin, is the best approach \cite{bock2004methods}. We find NDIGO and CGMM perform very similarly in this experiment, outperforming other methods and approaching KDE-plugin.
\begin{figure}[htb]
\centering
\begin{minipage}{0.33\textwidth}
        \centering
    \includegraphics[width=\textwidth]{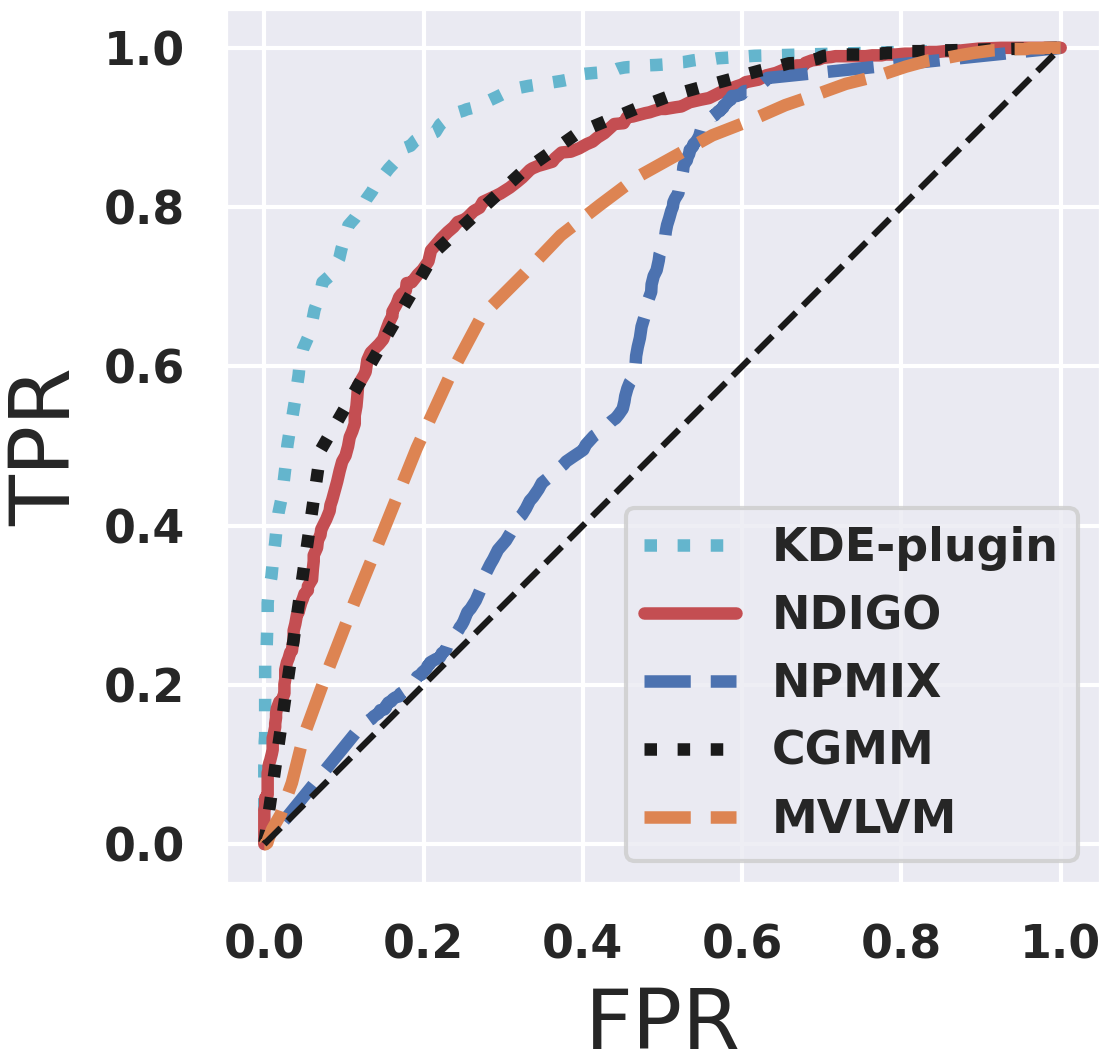}
    \caption{Receiver operating characteristic for MAGIC gamma ray detection dataset.}
    \label{fig:roc}
\end{minipage}
\hspace{0.03\linewidth}
\begin{minipage}{0.6\textwidth}
\centering
\captionsetup{type=table} 
  \caption{Russian-troll talking points learned from Twitter dataset. $^*$ Censored (racial epithet)}\label{tab:russian-troll}
  \centering
  \begin{tabular}{ll}
    \toprule
    Topic     & Selected Top 10 Words \\
    \midrule
    1 &  dead, man, kill, missing, families, young\\
    2 & make, good, better, enough, yet, even, get     \\
    3 & politics, inside, news, local, police, new, state   \\
    4 & trial, a$^*$, gentrified, wk, deport, b$^*$\\
    5 & businesses, competitive, strength, people, white \\
    \bottomrule
  \end{tabular}
  \\
  \vspace{5pt}
    \caption{Mean and standard deviation of topic coherence on Twitter dataset over five experiments.}\label{tab:coherence}
  \centering
  \begin{tabular}{llll}
    \toprule
    NDIGO & LF-DMM & GPU-DMM \\
    \midrule
    $0.456 \pm 0.069$ & $0.493 \pm 0.018$ & $0.435 \pm 0.009$ \\
    \bottomrule
  \end{tabular}
\end{minipage}
\end{figure}

We applied NDIGO to topic modeling on the Twitter dataset. Results are shown in Tables \ref{tab:russian-troll} and \ref{tab:coherence}. Details of data preprocessing are deferred to Appendix \ref{a:exp}. After preprocessing, the dataset consisted of $1,691,081$ pairs of $10$-dimensional embedded words where each element in a pair comes from the same tweet. Algorithms for competing methods, as described by their respective authors, could not scale to this experiment. Therefore, we compare to recent methods designed for continuous topic modeling of short texts: LF-DMM \cite{nguyen2015improving}, and GPU-DMM\cite{li2016topic} as implemented by Qiang et al. \cite{qiang2018STTP}.\footnote{available online: \url{https://github.com/qiang2100/STTM}} A selection of the top $10$ words of topics uncovered by NDIGO is given in Table \ref{tab:russian-troll}. We find that the discovered topics correspond well to other analyses of the dataset \cite{linvill2020troll}. Using topic coherence (pointwise mutual information) as an evaluation metric \cite{newman2010automatic}, we observe that NDIGO is competitive with the competing methods.

\section{Conclusion}
In this work we introduced a novel variant of the kernel density estimator that yields consistent estimates of any identifiable nonparametric mixture model from grouped observations. We established an oracle inequality for weighted kernel density estimators, and a general consistency result for estimators of the form $q_{w,\a}$. Namely, consistent estimation of $q$ implies consistent estimates of the underlying components when the NoMM is identifiable. In the case of $N=2$, we offer an efficient algorithm and demonstrate its effectiveness on several datasets where traditional approaches fail. Additionally, we show our approach has practical applications in topic modeling with very small documents and nuclear source detection.


\bibliographystyle{IEEEtran}
\bibliography{IEEEabrv,NPMM_arXiv.bib}

\begin{appendices}

\section{Additional Experimental Details} \label{a:exp}
In this section we provide details of the Twitter experiment and out-of-sample results for experiments on synthetic datasets.

\subsection{Out-of-sample Results}
Here we provide out-of-sample results for the synthetic experiments shown in the main paper. These are shown in Table \ref{tab:out-of-sample}. We make the realistic assumption that pair information is not available for out-of-sample data. We generate a test dataset $20\%$ the size of the training set according to the distribution of the training data.

\begin{table}[htb!]
    \centering
  \caption{Out-of-sample ARI and standard deviation over $10$ runs on synthetic datasets.}\label{tab:out-of-sample}
  \resizebox{\columnwidth}{!}{
  \begin{tabular}{llllll}
    \toprule
    Dataset & NDIGO & CSC & NPMIX & CGMM & MVLVM \\
    \midrule
    Overlapping moons & $0.705 \pm 0.091$ & $0.405 \pm 0.283$ & $0.131 \pm 0.075$ & $0.002 \pm 0.011$ & $0.593 \pm 0.229$\\
    Olympic Rings & $0.607 \pm 0.058$ & $0.387 \pm 0.033$ & $0.367 \pm 0.063$ & $0.127 \pm 0.012$ & $0.290 \pm 0.053$\\
    Half-disks & $0.221 \pm 0.036$ & $0.215 \pm 0.038$ & $0.102 \pm 0.095$ & $0.185 \pm 0.076$ & $0.127 \pm 0.062$ \\
    \bottomrule
  \end{tabular}
  }
\end{table}

\subsection{Preprocessing of Twitter Dataset}
Twitter dataset is publicly available through FiveThirtyEight.\footnote{ \url{https://github.com/fivethirtyeight/russian-troll-tweets}} The data consist of tweets, from a variety Russian-troll twitter accounts, tweeted between $2015$ and $2018$. We considered all tweets from 2016, a total of $878,878$. We pre-processed the tweets by removing stop words, punctuation, and hyperlinks, followed by a lemmatization step and the removal of any words that were not contained in the vocabulary of the six billion token GloVe word vectors \cite{pennington2014glove}. For completeness, we mention that lemmatization is a common pre-processing step in natural language processing that removes inflectional differences from words by mapping each inflection to a common base form called the lemma. For example, lemmatization will map each of the words dog, dogs, dog's, dogs', and doggy to the word dog. For each tweet, we paired the constituent words uniformly at random without replacement, resulting in $1,691,081$ pairs of words where the words of a given pair come from the same tweet. No other information was retained. We emphasize that a given word from a given tweet will \emph{not} be assigned to more than one pair. However, if a given word appears in multiple tweets (which may not all be about the same topic), it will show up in multiple pairs.

For the embedding step we performed PCA on the pre-trained GloVe $50$-dimensional embeddings to obtain $10$-dimensional vectors which were used to encode the paired words. The kernel centers were obtained by running mini-batch $k$-means with $R=200$ on a uniform random sample of $10,000$ of the $1,691,081$ word pairs, from which the matrix $G$ was also calculated. We then trained on the full $1,691,081$ word pairs, which are utilized in mini-batches through the matrix $C^{(t)}$. 

Algorithms for competing methods, as described by their respective authors, could not scale to this experiment. Therefore, we compare to recent methods designed for continuous topic modeling of short texts: LF-DMM \cite{nguyen2015improving}, and GPU-DMM\cite{li2016topic} as implemented by Qiang et al. \cite{qiang2018STTP}.\footnote{\url{https://github.com/qiang2100/STTM}} LF-DMM and GPU-DMM were trained on the same preprocessed data as NDIGO, where each of the $1,691,081$ word pairs is considered a unique document. Each of these methods were run with default hyperparameters, as described in the documentation for GPU-DMM and LF-DMM.$^2$ After training, UCI topic coherence \cite{newman2010automatic} (which measures pointwise mutual information) was used to evaluate performance. UCI topic coherence uses a reference dataset to estimate word co-occurrence probabilities, which is more robust in the short text setting as very common words in a given topic may never be observed to co-occur. A recent Wikipedia article dump was used for the reference dataset, and is provided with our code.

\section{Optimization Details} \label{a:opt}
In this section we provide details of our algorithm, including the form of the objective function and initialization, for both the full problem and the coreset approach.

Recall the expression for the ETISE
\begin{align*}
    \hat J(w,\a) &\triangleq \int q_{w,\a}^2(x,x') dx dx' 
    -2\sum_{m=1}^M \sum_{r=1}^n \sum_{r'=1}^2 \sum_{s=1}^n \sum_{s'=1}^2 w_m  \a_{mrr'} \a_{mss'} \hat h(r,r',s,s'), \label{ETISE}
\end{align*}
where
\begin{align*}
    \hat h(r,r',s,s') &\triangleq \begin{cases}
    \hat h_{\text{LTO}}(r,r',s,s'), & r \neq s \\
    \hat h_{\text{LOO}}(r,r',s'), & r=s
    \end{cases} \\
    \hat h_{\text{LOO}}(r,r',r'') &\triangleq \frac{1}{n-1} \sum_{i \in [n] \backslash \{r\}}k_\sigma(x_{i,1},x_{r,r'})k_\sigma(x_{i,2},x_{r,r''})\\
    \hat h_{\text{LTO}}(r,r',s,s') &\triangleq \frac{1}{n-2} \sum_{i \in [n] \backslash \{r,s\}}k_\sigma(x_{i,1},x_{r,r'})k_\sigma(x_{i,2},x_{s,s'}).
\end{align*}
For ease of computation, we will rewrite $\hat J(w,\a)$ in terms of matrix operations. In what follows, we assume that $\widetilde k_\sigma(z_r,z_u) \coloneqq \int k_\sigma(x,z_r) k_\sigma(x,z_u) dx$ has a closed-form expression or can otherwise be computed efficiently. Some examples \cite{kim2012robust} are give in Table \ref{tab:kernel}.

\begin{table}[htb!]\caption{Some popular kernel functions and their associated $\widetilde k$. Here $\left \lVert \cdot\right \rVert_2$ is the Euclidean norm.} \label{tab:kernel}
\begin{center}
\resizebox{\columnwidth}{!}{
\begin{tabular}{ |c|c|c| } 
 \hline
 Kernel & $k_\sigma(x, x')$ & $\widetilde k_\sigma(x,x')$ \\
 \hline
 \hline
 Gaussian & $\left(\frac{1}{\sqrt{2\pi}\sigma}\right)^d \exp \left( -\frac{\left \lVert x-x'\right \rVert_2^2}{2\sigma^2} \right)$ & $k_{\sqrt{2}\sigma}(x,x')$ \\ 
 Cauchy & $\left(\frac{1}{\sqrt{\pi}\sigma}\right)^d \left( \frac{\Gamma((1+d)/2)}{\Gamma(1/2)} \right) \left(\frac{\sigma^2 + \left \lVert x-x'\right \rVert_2^2}{\sigma^2} \right)^{-\frac{1+d}{2}}$ & $k_{2\sigma}(x,x')$ \\ 
 
 Laplacian & $\frac{c_d}{\sigma^d}\exp \left( -\frac{\left \lVert x-x'\right \rVert_1}{\sigma} \right)$ & $\frac{1}{(4\sigma)^d}\prod_{l=1}^d\left(\frac{\sigma + |x_l-x_l'|}{\sigma}\right)\exp \left( -\frac{\left \lVert x-x'\right \rVert_1}{\sigma} \right)$ \\
 \hline
\end{tabular}
}
\end{center}
\end{table}

\subsection{Full Optimization Problem}
We begin by examining the first term of $\hat J(w,\a)$ 
\begin{align*}
    &\int q_{w,\a}(x,x')^2 dx dx' = \int \Bigg(\sum_m w_m \sum_r \sum_{r'} \a_{m,r,r'} k_\sigma(x,x_{r,r'}) \sum_s \sum_{s'} \a_{m,s,s'} k_\sigma(x,x_{s,s'}) \Bigg) \\ 
    & \qquad \qquad \qquad \qquad \times \Bigg(\sum_j w_j \sum_u \sum_{u'} \a_{j, u,u'} k_\sigma(x,x_{u,u'}) \sum_v \sum_{v'} \a_{j, v,v'} k_\sigma(x,x_{v,v'}) \Bigg) dxdx' \\
    &= \sum_{m,j} w_m w_j \sum_{r,r',u,u'} \sum_{s,s',v,v'} \a_{m,r,r'} \a_{m,s,s'} \a_{j, u,u'} \a_{j, v,v'} \\
    &\qquad \qquad \qquad \qquad \times \int k_\sigma(x,x_{r,r'}) k_\sigma(x,x_{u,u'}) dx \int k_\sigma(x',x_{s,s'}) k_\sigma(x',x_{v,v'}) dx' \\
    &= \sum_{m,j} w_m w_j \sum_{r,r',u,u'} \a_{m,r,r'} \a_{j, u,u'} \widetilde k_\sigma(x_{r,r'}, x_{u,u'}) \sum_{s,s',v,v'} \a_{m,s,s'} \a_{j, v,v'} \widetilde k_\sigma(x_{s,s'}, x_{v,v'}) \\
    &= \sum_{m,j} w_m w_j \Bigg( \a_m' G a_j \Bigg)^2,
\end{align*}
where $\times$ in the first line is scalar multiplication and $G$ is the kernel matrix of the data and is given by

\begin{equation*}
     G  = \left[ \begin{array}{cccccc}
            \widetilde k_\sigma(x_{1,1},x_{1,1}) &  \widetilde k_\sigma(x_{1,1},x_{1,2}) & \cdots & \cdots &  \widetilde k_\sigma(x_{1,1},x_{n,1}) & \widetilde k_\sigma(x_{1,1},x_{n,2}) \\
          \widetilde k_\sigma(x_{1,2},x_{1,1}) &  \widetilde k_\sigma(x_{1,2},x_{1,2}) & \cdots & \cdots &  \widetilde k_\sigma(x_{1,2},x_{n,1}) & \widetilde k_\sigma(x_{1,2},x_{n,2}) \\
          \vdots & \vdots & \ddots &  \ddots& \vdots & \vdots\\
          \vdots & \vdots & \ddots &  \ddots& \vdots & \vdots\\
          \widetilde k_\sigma(x_{n,1},x_{1,1}) &  \widetilde k_\sigma(x_{n,1},x_{1,2}) & \cdots & \cdots &  \widetilde k_\sigma(x_{n,1},x_{n,1}) & \widetilde k_\sigma(x_{n,1},x_{n,2}) \\
          \widetilde k_\sigma(x_{n,2},x_{1,1}) &  \widetilde k_\sigma(x_{n,2},x_{1,2}) & \cdots & \cdots &  \widetilde k_\sigma(x_{n,2},x_{n,1}) & \widetilde k_\sigma(x_{n,2},x_{n,2}) \\
     \end{array}\right].
\end{equation*}

Examining the second term of the ETISE yields
\begin{align*}
    &\mathbb{E}_q\left[ q_{w,\a} \right] \approx \sum_{m=1}^M \sum_{r=1}^n \sum_{r'=1}^2 \sum_{s=1}^n \sum_{s'=1}^2 w_m  \a_{m,r,r'} \a_{m,s,s'} \hat h(r,r',s,s') \\
    &= \begin{cases} \frac{1}{n-2}  \sum_{m}\sum_{r,r'}\sum_{s,s'} w_m  \a_{m,r,r'} \a_{m,s,s'} \sum_{i \in [n] \backslash \{r,s\}}k_\sigma(x_{i,1},x_{r,r'})k_\sigma(x_{i,2},x_{s,s'}), & r\neq s\\
    \frac{1}{n-1}  \sum_{m}\sum_{r,r'}\sum_{s,s'}w_m  \a_{m,r,r'} \a_{m,s,s'} \sum_{i \in [n] \backslash \{r\}}k_\sigma(x_{i,1},x_{r,r'})k_\sigma(x_{i,2},x_{r,s'}), & r=s
    \end{cases} \\
    &= \sum_{m=1}^M w_m \left(\a_m'C\a_m\right),
\end{align*}
where $C$ is given by

\begin{equation*}
    C =  \left[ \begin{array}{cccccc}
            \hat h(1,1,1,1) &  \hat h(1,1,1,2) & \cdots & \cdots &  \hat h(1,1,n,1) & \hat h(1,1,n,2) \\
          \hat h(1,2,1,1) &  \hat h(1,2,1,2) & \cdots & \cdots &  \hat h(1,2,n,1) & \hat h(1,2,n,2) \\
          \vdots & \vdots & \ddots &  \ddots& \vdots & \vdots\\
          \vdots & \vdots & \ddots &  \ddots& \vdots & \vdots\\
           \hat h(n,1,1,1) &  \hat h(n,1,1,2) & \cdots & \cdots &  \hat h(n,1,n,1) & \hat h(n,1,n,2) \\
          \hat h(n,2,1,1) &  \hat h(n,2,1,2) & \cdots & \cdots &  \hat h(n,2,n,1) & \hat h(n,2,n,2) \\
     \end{array}\right].
\end{equation*}
The diagonal blocks of size two of the matrix $C$ use the leave one out estimator, while the other entries use the leave two out estimator.

\subsection{Coreset Approach}
In the full problem the matrices $G,C \in \R^{2n\times 2n}$ grow linearly with the data. This can make the proposed optimization problem \eqref{eq:generic_problem} costly to solve, as the complexity of gradient calculations are quadratic in the dimensions of $G,C$. Additionally, the complexity of evaluating out-of-sample data is quadratic in $n$ for general KDEs. The motivation of the coreset approach is to reduce this complexity.

KDEs traditionally center kernels at the location of each observation, i.e., $k_\sigma(\cdot, x_{i,i'})$, where we call $x_{i,i'}$ the kernel center. Rather than constraining the wKDE to have kernels centered at the observations, we can formulate the optimization problem with $R$ kernel centers $z_r \in \R^d$ for some suitably chosen $z_r$. Additionally, choosing $R\ll n$ will substantially reduce the complexity of gradient calculations and out-of-sample evaluation. The collection of kernel centers $z_r$ will be our coreset. We don't provide guarantees for the optimality of any particular coreset. The coreset could potentially be chosen as the cluster centers output by some clustering algorithm, some suitable subset of the data, or perhaps via some more principled scheme. In all of our experiments, we chose the coreset to be cluster centers output by mini-batch $k$-means, where the number of clusters was chosen to be $R > M$. 

For the coreset approach, the ETISE has the same form but the matrices $G$ and $C$ have the form
\begin{align*}
     G  = &\left[ \begin{array}{ccccccccc}
          \widetilde k_\sigma(z_{1},z_{1}) &  \cdots & \cdots &  \widetilde k_\sigma(z_{1},z_{R}) \\ 
          \vdots & \ddots &  \ddots& \vdots \\
          \vdots & \ddots & \ddots & \vdots \\
          \widetilde k_\sigma(z_{R},z_{1}) & \cdots & \cdots & \widetilde  k_\sigma(z_{R},z_{R}) 
     \end{array}\right], \\
    C = \frac{1}{n}\sum_{i=1}^nC_i  = \frac{1}{n}\sum_{i=1}^n &\left[ \begin{array}{ccccccccc}
          k_\sigma(x_{i,1},z_{1})k_\sigma(x_{i,2},z_{1}) & \cdots & \cdots &  k_\sigma(x_{i,1},z_{1})k_\sigma(x_{i,2},z_{R}) \\
          \vdots & \ddots & \ddots & \vdots \\
          \vdots & \ddots & \ddots & \vdots \\
          k_\sigma(x_{i,1},z_{R})k_\sigma(x_{i,2},z_{1}) & \cdots & \cdots &  k_\sigma(x_{i,1},z_{R})k_\sigma(x_{i,2},z_{R})
     \end{array}\right]
\end{align*}
This is derived in the same way as the full problem, replacing the kernel centers $x_{r,r'},x_{s,s'}$ with the coreset $z_r$, and using the "leave-none-out" estimator in place of the LOO/LTO estimator $\hat h$.

\subsection{Algorithm}
Though the problem \eqref{optprob} is nonconvex, we observe that a properly initialized alternating projected stochastic gradient descent (APSGD) procedure produces good solutions in practice. Pseudocode for the APSGD algorithm for solving \eqref{optprob} is given in Algorithm \ref{APSGD}. We mention that the projections $\Pi_\Delta$ are onto the probability simplex, a decaying step size $\eta^{(t)}$ is used, and stochasticity is introduced via the matrix $C^{(t)}$, which is a mini-batch version of $C$ defined by $C^{(t)}_{a,b} = \frac{1}{|\Omega^{(t)} \backslash \{a,b\}|}\sum_{i \in |\Omega^{(t)} \backslash \{a,b\}|} k_\sigma(X_{i,1},X_{\lfloor\frac{a}{2}\rfloor, a\text{ mod } 2})k_\sigma(X_{i,2},X_{\lfloor\frac{b}{2}\rfloor, b\text{ mod } 2})$,
where $\Omega^{(t)}$ is the index set corresponding to the $t^{th}$ mini-batch.

\begin{algorithm}[htb!]
\caption{Alternating Projected SGD} \label{APSGD}
\begin{algorithmic}[1]
\State \textbf{init:} $\a^{(0)}, w^{(0)}, \eta^{(0)}$
 \Procedure{APSGD}{$\a^{(0)}, w^{(0)}, \eta^{(0)}$} 
    \State{Form $G$}
    \For{$t=1,2,\dots$}
    \State{Take a minibatch of paired observations indexed by $\Omega^{(t)}$}
    \State{Form $C^{(t)}$ from the minibatch according to the definition of $C$}
    \State{$w^{(t)} = \Pi_\Delta(w^{(t-1)} - \eta^{(t)}\nabla_w \hat J(w^{(t-1)},\a^{(t-1)}))$} 
    \For{$j=1,\dots,M$}
    \State{$\a_j^{(t)} = \Pi_\Delta(\a_j^{(t-1)} - \eta^{(t)}\nabla_{\a_j}\hat J(w^{(t)}, \a_j^{(t-1)}))$}
    \EndFor{}
    \EndFor{}
 \EndProcedure
 \end{algorithmic}
\end{algorithm}

\subsection{Spectral Initialization}
We adopt a spectral initialization scheme. First, the initialization is presented for the full problem, then we adapt it to the coreset approach. The idea here is, given some estimator of $q$, lets say $\tilde{q}$, to find a low rank approximation of $\tilde{q}$
$$
    \tilde{q}(x,y) \approx \sum_{i=1}^M \lambda_i \psi_i(x)\psi_i(y)
$$
and then to use $\lambda_i$ and $\psi_i$ as starting points for our mixture weights and components. We do this by using the full grouped sample data as an estimate of $q$ which we transform into a linear operator and decompose using a functional eigenvector decomposition.

We begin with a standard KDE applied to our full samples using a product kernel:
\begin{equation*}
    f_{\sigma}(y,y') = \frac{1}{2n}\sum_{i=1}^n  k_{\sigma}(y, x_{i,1})  k_\sigma(y', x_{i,2}) + k_{\sigma}(y, x_{i,2}) k_\sigma(y', x_{i,1}).
\end{equation*}
Note that we include centers at both $(x_{i,1},x_{i,2})$ and $(x_{i,2},x_{i,1})$ so our KDE is symmetric in $y,y'$.

By Lemmas 5.1 and 8.2 of Vandermeulen and Scott \cite{vandermeulen_operator_2019}, $f_\sigma$ can be viewed as an element of a tensor product space $L^2(\R^d) \otimes L^2(\R^d)$ as follows
$$
    f_{\sigma} = \frac{1}{2n}\sum_{i=1}^n  k_{\sigma}(\cdot, x_{i,1}) \otimes k_\sigma(\cdot, x_{i,2}) + k_{\sigma}(\cdot, x_{i,2})\otimes k_\sigma(\cdot, x_{i,1}).
$$
By the Lemmas referenced above, there is a unitary transformation on the KDE $f_\sigma$ such that it can be viewed as a linear operator $T:L^2(\R^d) \to L^2(\R^d)$ given by 
\begin{equation*}
    T(g) \coloneqq \sum_{i=1}^n k_\sigma(\cdot,x_{i,1})\langle k_\sigma(\cdot,x_{i,2}), g(\cdot) \rangle_{L^2}  + k_\sigma(\cdot,x_{i,2})\langle k_\sigma(\cdot,x_{i,1}),  g(\cdot) \rangle_{L^2}, \quad \forall g \in L^2(\R^d)
\end{equation*}
which is symmetric since it includes $k_\sigma(\cdot,x_{i,1})\langle k_\sigma(\cdot,x_{i,2}), g(\cdot) \rangle_{L^2}$ and $k_\sigma(\cdot,x_{i,2})\langle k_\sigma(\cdot,x_{i,1}),  g(\cdot) \rangle_{L^2}$ terms. We have removed the $1/(2n)$ coefficient since it will not affect the spectral decomposition.
For any $g \in L^2$ the quantity $\langle k_\sigma(\cdot,x_{i,i'}), g(\cdot) \rangle_{L^2}$ will be a finite scalar, so $T(g)$ will be a linear combination of the $k_\sigma(\cdot,x_{i,i'})$. Therefore, eigenvectors of the above linear operator will have the form $g(\cdot) = \sum_{j,j'}\beta_{j,j'}k_\sigma(\cdot,x_{j,j'})$ since $T$ applied to any vector must lie in the span of $k_\sigma\left(\cdot, x_{j,j'} \right)$. Evaluating $T$ on vectors of this form (not necessarily an eigenvector) will yield
\begin{align*}
    T(g) &\coloneqq \sum_{i=1}^n \Big\{k_\sigma(\cdot,x_{i,1})\langle k_\sigma(\cdot,x_{i,2}), \sum_{j,j'}\beta_{j,j'}k_\sigma(\cdot,x_{j,j'}) \rangle_{L^2} \\
    &\qquad \qquad+ k_\sigma(\cdot,x_{i,2})\langle k_\sigma(\cdot,x_{i,1}),  \sum_{j,j'}\beta_{j,j'}k_\sigma(\cdot,x_{j,j'}) \rangle_{L^2} \Big\} \\
    &= \sum_{i}\zeta_{i,1}k_\sigma (\cdot,x_{i,1}) + \zeta_{i,2}k_\sigma (\cdot,x_{i,2})
\end{align*}
where $\zeta_{i,1} = \sum_{j,j'}\beta_{j,j'} \tilde{k}_\sigma(x_{i,2},x_{j,j'})$, $\zeta_{i,2} = \sum_{j,j'}\beta_{j,j'} \tilde{k}_\sigma(x_{i,1},x_{j,j'})$ and $\tilde{k}_\sigma(y,y')$ is defined in Appendix \ref{a:opt}. 

Define the ordering of the elements of $\beta$ and $\zeta$ by 
\begin{align*}
    \beta &= [\beta_{1,1}, \beta_{1,2}, \beta_{2,1}, \beta_{2,2}, \dots,  \beta_{n,1}, \beta_{n,2}]',\\
    \zeta &= [\zeta_{1,1}, \zeta_{1,2}, \zeta_{2,1}, \zeta_{2,2}, \dots,  \zeta_{n,1}, \zeta_{n,2}]'.
\end{align*}
Then we have
\begin{equation*}
    \zeta = \bar G\beta,
\end{equation*}
where
\begin{equation*}
    \bar G = \left[ \begin{array}{cccccccc}
            \widetilde k_\sigma(x_{1,2},x_{1,1}) &  \widetilde k_\sigma(x_{1,2},x_{1,2})
            &\widetilde k_\sigma(x_{1,2},x_{2,1}) &  \cdots &  \widetilde k_\sigma(x_{1,2},x_{n,1}) & \widetilde k_\sigma(x_{1,2},x_{n,2}) \\
            \widetilde k_\sigma(x_{1,1},x_{1,1}) &  \widetilde k_\sigma(x_{1,1},x_{1,2})
            &\widetilde k_\sigma(x_{1,1},x_{2,1}) &    \cdots &  \widetilde k_\sigma(x_{1,1},x_{n,1}) & \widetilde k_\sigma(x_{1,1},x_{n,2}) \\
            \widetilde k_\sigma(x_{2,2},x_{1,1}) &  \widetilde k_\sigma(x_{2,2},x_{1,2})
            &\widetilde k_\sigma(x_{2,2},x_{2,1}) &   \cdots &  \widetilde k_\sigma(x_{2,2},x_{n,1}) & \widetilde k_\sigma(x_{2,2},x_{n,2}) \\
            \widetilde k_\sigma(x_{2,1},x_{1,1}) &  \widetilde k_\sigma(x_{2,1},x_{1,2})
            &\widetilde k_\sigma(x_{2,1},x_{2,1}) 
            & \cdots &  \widetilde k_\sigma(x_{2,1},x_{n,1}) & \widetilde k_\sigma(x_{2,1},x_{n,2}) \\
          \vdots & \vdots & \vdots &  \ddots& \vdots & \vdots\\
          \widetilde k_\sigma(x_{n,2},x_{1,1}) &  \widetilde k_\sigma(x_{n,2},x_{1,2}) & \widetilde k_\sigma(x_{n,2},x_{2,1}) & \cdots &  \widetilde k_\sigma(x_{n,2},x_{n,1}) & \widetilde k_\sigma(x_{n,2},x_{n,2}) \\
          \widetilde k_\sigma(x_{n,1},x_{1,1}) &  \widetilde k_\sigma(x_{n,1},x_{1,2}) & \widetilde k_\sigma(x_{n,1},x_{2,1}) & \cdots &  \widetilde k_\sigma(x_{n,1},x_{n,1}) & \widetilde k_\sigma(x_{n,1},x_{n,2}) \\
     \end{array}\right]
\end{equation*}
is a row permutation of $G$ defined for the full problem, obtained by exchanging rows corresponding to the first and second elements of each paired observation. The takeaway is that the coefficients of the eigenvectors of $T$ are given by the right eigenvectors of $\bar G$. In particular, we will take the right eigenvectors of $\bar G$ corresponding to the $M$ largest real eigenvalues with the intuition that they will capture the dominant modes of $T$. Note that these eigenvectors will contain real-valued entries by the spectral theorem since $T$ is symmetric. We call these eigenvectors the first, second, and so on. It should be noted that $\bar G$ is not a symmetric matrix and so eigenvectors should be found according to, for example, the power iteration or orthogonal iteration. In general the eigenvectors of $ \bar G$ will have negative entries and not sum to one, so we project these eigenvectors onto the probability simplex to obtain the non-negative weights for initialization where we take $\a_1$ to the be the projection of the first eigenvector of $\bar G$, $\a_2$ to the be the second, and so on. For the initial $w_i$, We take $w_1$ to be the first eigenvalue of $\bar G$, $w_2$ to be the second and so on, projecting the resulting $w = [w_1, w_2, \dots, w_M]'$ onto the probability simplex.

\paragraph{Coreset Approach Initialization} Initialization for the coreset approach is similar. Since we assume no relationship between $z_i$, we don't have the same notion of using paired kernel centers even though our data is still paired. However, we can still write the KDE using a product kernel over the coreset as
\begin{equation*}
    f_{\sigma}(y,y') = \frac{1}{R}\sum_{i=1}^R  k_{\sigma}(y, z_i)  k_\sigma(y', z_i).
\end{equation*}
This KDE is symmetric since the same kernel center is used in each term of the product kernel. Again appealing to Lemmas 5.1 and 8.2 of Vandermeulen and Scott \cite{vandermeulen_operator_2019}, $f_\sigma$ can be viewed as an element of a tensor product space $L^2(\R^d) \otimes L^2(\R^d)$ as follows
$$
    f_{\sigma} = \frac{1}{R}\sum_{i=1}^R  k_{\sigma}(\cdot, z_i) \otimes k_\sigma(\cdot, z_i).
$$
By the Lemmas referenced above, there is a unitary transformation on the KDE $f_\sigma$ such that it can be viewed as a linear operator $T:L^2(\R^d) \to L^2(\R^d)$ given by
\begin{equation*}
    T(g) = \sum_{i=1}^R k_\sigma(\cdot, z_i)\langle k_\sigma(\cdot,z_i), g(\cdot) \rangle_{L^2}, \quad \forall g \in L^2.
\end{equation*}
By the same argument as the full problem, the eigenvectors of the above operator have the form $g(\cdot) = \sum_{j}\beta_{j}k_\sigma(\cdot,z_j)$. Applying $T$ to a vector of this form (not necessarily an eigenvector), we have
\begin{align*}
    T(\sum_{j}\beta_{j}k_\sigma(\cdot,z_j)) &= \sum_{i=1}^R k_\sigma(\cdot, z_i) \langle k_\sigma(\cdot, z_i), \sum_{j}\beta_{j}k_\sigma(\cdot,z_j) \rangle_{L^2} \\
    &= \sum_{i=1}^R \zeta_i k_\sigma(\cdot, z_i),
\end{align*}
where $\zeta_i = \sum_j \beta_j k_\sigma(\cdot, z_i)$. In this setting we have the standard ordering, $\beta = [\beta_1, \beta_2, \dots, \beta_R]'$ and $\zeta = [\zeta_1, \zeta_2, \dots, \zeta_R]'$. In matrix form, the relationship between $\beta$ and $\zeta$ is given by
\begin{equation*}
    \zeta = G \beta,
\end{equation*}
where $G$ is as previously defined for the coreset approach. No row permutation is needed in this setting as both centers of our product kernel are the same. From this point, the initialization scheme is essentially the same as for the full problem, but using the eigenvectors and eigenvalues of $G$. One key difference is that $G$ \emph{is} a symmetric matrix, so the eigenvalues and eigenvectors of $G$ can be found using any standard solver.

\section{Proof and General Form of Theorem \ref{oe}} \label{a:t1}
Theorem \ref{oe} was presented in the main paper for groups of size $N=2$. Here, we provide the proof for groups of size two, as well as the proof for the more general case of arbitrary group size $N>2$. One tool we will use is Hoeffding's inequality for independent bounded random variables, which we state here for completeness.
\begin{theorem*}
Hoeffding's Inequality: Let $V_1, V_2, \dots, V_n$ be independent bounded random variables such that $a_i\leq V_i \leq b_i$ with probability one. If $S_n = \sum_{i=1}^n V_i$, then for all $t > 0$
\begin{equation*}
    P\left\{\Big|S_n - \mathbb{E}\{S_n\} \Big| \geq t\right\}\leq 2\exp\left\{-\frac{2t^2}{\sum_{i=1}^n (b_i-a_i)^2}\right\}.
\end{equation*}
\end{theorem*}

\subsection{Proof of Theorem \ref{oe}: Groups of Size Two}
We restate Theorem \ref{oe} for convenience.
\begin{manualtheorem}{\ref{oe}}
Let $\epsilon>0$ and set $\delta = 8(n^2-n)\exp\{-\frac{\sigma^{4d}(n-2)\epsilon^2}{8C_k^4}\} + 8n\exp\{-\frac{\sigma^{4d}(n-1)\epsilon^2}{8C_k^4}\}$. With probability at least $1-\delta$ the following holds: 
\begin{equation*}
    \left \lVert q - q_{\hat w, \hat \a} \right \rVert_2^2 \leq \  \inf_{w \in \Delta^M, \ \a \in \Delta_M^{2n}} \left \lVert q - q_{w, \a} \right \rVert_2^2 + \epsilon.
\end{equation*}
\end{manualtheorem}

\begin{proof}
Our goal is to bound $|J(w,a) - \hat J(w,a)|$ uniformly over $w \in \Delta^M$, $\a \in \Delta_M^{2n}$. Recall the following definitions
\begin{align*}
    h(r,r',s,s') &\coloneqq \int k_\sigma(x,x_{r,r'})k_\sigma(x',x_{s,s'})q(x,x')dxdx' \\
    \hat h(r,r',s,s') &\coloneqq \begin{cases}
    \hat h_{\text{LTO}}(r,r',s,s'), & r \neq s \\
    \hat h_{\text{LOO}}(r,r',s'), & r=s
    \end{cases} \\
    \hat h_{\text{LOO}}(r,r',r'') &\coloneqq \frac{1}{n-1} \sum_{i \in [n] \backslash \{r\}}k_\sigma(x_{i,1},x_{r,r'})k_\sigma(x_{i,2},x_{r,r''})\\
    \hat h_{\text{LTO}}(r,r',s,s') &\coloneqq \frac{1}{n-2} \sum_{i \in [n] \backslash \{r,s\}}k_\sigma(x_{i,1},x_{r,r'})k_\sigma(x_{i,2},x_{s,s'}).
\end{align*}
The use of the leave one out (LOO) and leave two out (LTO) estimators above is to ensure independence so that we will be able to apply Hoeffding's inequality. We have
\begin{align*}
    P_{q} \{ \sup_{\substack{w \in \Delta^M \\ \a_ \in \Delta_M^{2n} } }& | J(w,\a) - \hat J(w,\a) | > \frac{\epsilon}{2} \} \\
    &= P_{q} \Big\{ \sup_{\substack{w \in \Delta^M \\ \a_ \in \Delta_M^{2n}}} \Big| \sum_{m=1}^Mw_m\sum_{r=1}^{n}\sum_{s=1}^{n}\sum_{r'=1}^{2}\sum_{s'=1}^{2}\a_{m,r,r'}\a_{m,s,s'}h(r,r',s,s') \\   
    & \quad \quad \quad \quad \quad \quad -\sum_{m=1}^Mw_m\sum_{r=1}^{n}\sum_{s=1}^{n}\sum_{r'=1}^{2}\sum_{s'=1}^{2}\a_{m,r,r'}\a_{m,s,s'}\hat h(r,r',s,s')\Big|> \frac{\epsilon}{4} \Big\} \\
    &\leq P_{q} \Big\{ \sup_{\substack{w \in \Delta^M \\ \a_ \in \Delta_M^{2n} } }  \sum_{m}\sum_{r,s}\sum_{r',s'} w_m\a_{m,r,r'} \a_{m,s,s'} \Big|h(r,r',s,s')   -\hat h(r,r',s,s')\Big| > \frac{\epsilon}{4}\Big\} \\
    &\leq P_{q} \Big\{ \max_{r,s,r',s'} \Big|h(r,r',s,s') -\hat h(r,r',s,s')\Big| > \frac{\epsilon}{4}\Big\} \\
    &\leq \sum_{r,s}\sum_{r',s'}P_{q} \Big\{ \Big|h(r,r',s,s') -\hat h(r,r',s,s')\Big| > \frac{\epsilon}{4}\Big\}.
\end{align*}
The second step above is due to the triangle inequality, and the penultimate step is due to simplex constraints on $w,\a$. Let $k_i(r,r',s,s') \coloneqq k_\sigma(x_{i,1}, x_{r,r'})k_\sigma(x_{i,2}, x_{s,s'})$. Noting that $h(r,r',s,s') = \mathbb{E}_{(x_{i,1},x_{i,2})\sim q}\{k_i(r,r',s,s')\}$,
\begin{align*}
    &P_{q}\left\{ \big| h(r,r',s,s') - \hat h(r,r',s,s') \Big| > \frac{\epsilon}{4} \right\} \\
    &= \begin{cases}
    P_{q}\left\{ \Big| \mathbb{E}_{(x_{i,1},x_{i,2})\sim q}\{k_i(r,r',s,s')\} - \frac{1}{n-2} \sum_{i \in [n] \backslash \{r,s\}}k_i(r,r',s,s') \Big|>\frac{\epsilon}{4}\right\}, & r \neq s \\[10pt]
     P_{q}\left\{ \Big| \mathbb{E}_{(x_{i,1},x_{i,2})\sim q}\{k_i(r,r',s,s')\} - \frac{1}{n-1} \sum_{i \in [n] \backslash \{r\}}k_i(r,r',s,s') \Big|>\frac{\epsilon}{4}\right\}, & r = s
    \end{cases} \\
    &= \begin{cases}
    P_{q}\left\{ \Big| \frac{1}{n-2} \sum_{i \in [n] \backslash \{r,s\}} \mathbb{E}_{(x_{i,1},x_{i,2})\sim q}\{ k_i(r,r',s,s')\} - k_i(r,r',s,s') \Big|>\frac{\epsilon}{4}\right\}, & r \neq s \\[10pt]
     P_{q}\left\{ \Big| \frac{1}{n-1} \sum_{i \in [n] \backslash \{r\}}\mathbb{E}_{(x_{i,1},x_{i,2})\sim q}\{k_i(r,r',s,s')\} - k_i(r,r',s,s') \Big|>\frac{\epsilon}{4}\right\}, & r = s
    \end{cases} \\
     &= \begin{cases}
    P_{q}\left\{ \Big| \sum_{i \in [n] \backslash \{r,s\}} \mathbb{E}_{(x_{i,1},x_{i,2})\sim q}\{ k_i(r,r',s,s')\} - k_i(r,r',s,s') \Big|>\frac{(n-2)\epsilon}{4}\right\}, & r \neq s \\[10pt]
     P_{q}\left\{ \Big| \sum_{i \in [n] \backslash \{r\}}\mathbb{E}_{(x_{i,1},x_{i,2})\sim q}\{k_i(r,r',s,s')\} - k_i(r,r',s,s') \Big|>\frac{(n-1)\epsilon}{4}\right\}, & r = s
    \end{cases}
\end{align*}
The terms $k_i(r,r',s,s')$ are independent random variables due to use of the LOO/LTO estimator. By assumption, $0 \leq k_i(r,r',s,s')\leq C_k^2\sigma^{-2d}$ so the $k_i$ are bounded for fixed $\sigma>0$. We apply Hoeffding's inequality
\begin{align*}
    P_{q}\left\{ \big| h(r,r',s,s') - \hat h(r,r',s,s') \Big| > \frac{\epsilon}{4} \right\} &\leq \begin{cases}
    2\exp\{-\frac{2(n-2)^2\epsilon^2}{16(n-2)C_k^4\sigma^{-4d}}\}, & r \neq s \\[10pt]
    2\exp\{-\frac{2(n-1)^2\epsilon^2}{16(n-1)C_k^4\sigma^{-4d}}\}, & r = s
    \end{cases} \\
    &\leq \begin{cases}
    2\exp\{-\frac{\sigma^{4d}(n-2)\epsilon^2}{8C_k^4}\}, & r \neq s \\[10pt]
    2\exp\{-\frac{\sigma^{4d}(n-1)\epsilon^2}{8C_k^4}\}, & r = s
    \end{cases}. \numberthis \label{hoeffding}
\end{align*}
Substituting backward we obtain the desired upper bound
\begin{align*}
    P_{q} \{ \sup_{\substack{w \in \Delta^M \\ \a_ \in \Delta_M^{2n} } } | J(w,\a) - &\hat J(w,\a) | > \frac{\epsilon}{2} \} \leq \sum_{r,s}\sum_{r',s'}
    \begin{cases}
    2\exp\{-\frac{\sigma^{4d}(n-2)\epsilon^2}{8C_k^4}\}, & r \neq s \\[10pt]
    2\exp\{-\frac{\sigma^{4d}(n-1)\epsilon^2}{8C_k^4}\}, & r = s
    \end{cases} \\
          &= 8(n^2-n)\exp\{-\frac{\sigma^{4d}(n-2)\epsilon^2}{8C_k^4}\} + 8n\exp\{-\frac{\sigma^{4d}(n-1)\epsilon^2}{8C_k^4}\}.
\end{align*}
    
Letting $\delta = 8(n^2-n)\exp\{-\frac{\sigma^{4d}(n-2)\epsilon^2}{8C_k^4}\} + 8n\exp\{-\frac{\sigma^{4d}(n-1)\epsilon^2}{8C_k^4}\}$, we have
\begin{align} \label{sandwich}
    &J(w,\a) - \frac{\epsilon}{2} \leq \hat J(w,\a) \leq J(w,\a) + \frac{\epsilon}{2} \quad \forall w,\a
\end{align}
 with probability at least $1-\delta$. Thus, with probability at least $1-\delta$, for any $w \in \Delta^M,\a \in \Delta^{2n}_M$
\begin{align*}
    J(\hat w, \hat \a) &\leq \hat J(\hat w, \hat \a) + \frac{\epsilon}{2}\\
    &\leq \hat J(w, \a) + \frac{\epsilon}{2} \\
    &\leq J(w,\a) + \epsilon,
\end{align*}
where $\hat w, \hat \a$ are defined in \eqref{eq:generic_problem}. Then with probability at least $1-\delta$  
\begin{equation} \label{hatVinf}
    \hat J(\hat w, \hat \a) \leq \inf_{\substack{w \in \Delta^M \\ \a \in \Delta_M^{2n}}} J(w,\a) + \epsilon.
\end{equation}
Combining \eqref{hatVinf} with the definition of the ISE shows, with probability at least $1-\delta$,
\begin{equation*}
     \left \lVert q - q_{\hat w, \hat \a} \right \rVert_2^2 \leq \  \inf_{\substack{w \in \Delta^M \\ \a \in \Delta_M^{2n}}} \left \lVert q - q_{w, \a} \right \rVert_2^2 + \epsilon. 
\end{equation*}
\end{proof}

\subsection{Theorem \ref{oe}: Arbitrary Group Size}
\subsubsection{Preliminaries} \label{arb_group_setup}
Before beginning the proof, we start by redefining $q$, $q_{w,a}$, $J$, and $\hat J$ for arbitrary group size. Once this is done, the proof will follow the same basic steps as the proof for groups of size two. 

Suppose we change the problem setup only in the size of the grouped observations. Consider grouped observations of size $N$. Consider a set of $n$ grouped observations $\bx_1,\dots,\bx_n$ with $\bx_i=(x_{i,1},\dots,x_{i,N}) \in \mathbb{R}_N^d \coloneqq \underbrace{\mathbb{R}^d \times \dots \times \mathbb{R}^d}_N$ drawn i.i.d. from 
\begin{equation}
q(y_1, y_2, \dots, y_N) = \sum_{m=1}^M w^*_m p^*_m(y_1)p^*_m(y_2) \dots p^*_m(y_N), \quad y_1,y_2,\dots,y_N \in \R^d. \label{q_a}
\end{equation}
Similar to the paired observation setting, a wKDE in this setting will have the form
\begin{equation*}
    p(y;\theta) = \sum_{r=1}^n\sum_{r'=1}^N \theta_{r,r'}k_\sigma(y,x_{r,r'}).
\end{equation*}
We may write the corresponding estimator of $q$
\begin{align*}
   q_{w,\a}(y_1,y_2,\dots,y_N) &= \sum_{m=1}^Mw_m p(y_1;\a_m)p(y_2;\a_m) \dots p(y_N;\a_m)
\end{align*}
where $\a_m = [\a_{m,1,1} \dots \a_{m,1,N} \dots \dots \a_{m,n,1} \dots \a_{m,n,N}]' \in \Delta^{Nn} \ $ for $ \ m=1,\dots,M$, with $\a_{m,r,r'}$ corresponding to the weight of the kernel centered at $x_{r,r'}$ in the estimate of the $m^{th}$ mixture component.

In what follows we use $\sum_{r,r'} \coloneqq \sum_{r_1,r_1'} \dots \sum_{r_N,r_N'}$ to ease notation. Similar to the paired sample case, we define
\begin{align*}
    J(w,\a) &\coloneqq \int q_{w,\a}^2(y_1,\dots,y_N) dy_1 \dots dy_N 
    -2\sum_{m,r,r'} \left( \prod_{i \in [N] } w_m \a_{m,r_i,r_i'}\right) h(r_1,r_1', \dots, r_N,r_N') \\
    \hat J(w,\a) &\coloneqq \int q_{w,\a}^2(y_1,\dots,y_N) dy_1 \dots dy_N 
    -2\sum_{m,r,r'} \left( \prod_{i \in [N] } w_m \a_{m,r_i,r_i'}\right) \hat h(r_1,r_1', \dots, r_N,r_N'),
\end{align*}
where
\begin{align*}
    h(r_1,r_1', \dots, r_N,r_N') &\coloneqq \int k_\sigma(y_1,x_{r_1,r_1'})\dots k_\sigma(y_N,x_{r_N,r_N'})q(y_1, 
    \dots,y_N)dy_1 \dots dy_N, \\
    \hat h \coloneqq \hat h_{\text{LNO}}(r_1,r_1', \dots, r_N,r_N') &\coloneqq \frac{1}{n-N} \sum_{i \in [n] \backslash L(r_1,r_2, \dots, r_N)}k_\sigma(x_{i,1},x_{r_1,r_1'})\dots k_\sigma(x_{i,N},x_{r_N,r_N'}),
\end{align*}
where $L(r_1,r_2, \dots, r_N)$ is any subset of $[n]$ containing $\{r_1,r_2, \dots, r_N\}$ and having cardinality $N$. If $r_1,r_2,\dots,r_N$ are not distinct, the additional indices can be chosen arbitrarily. For simplicity, we use a leave-N-out (LNO) estimator $\hat h_{\text{LNO}}$ rather than a hybrid estimator like we used in the case of paired observations. As in the paired observation setting, we define
\begin{equation*}
    (\hat w, \hat \a) \coloneqq \ \argmin_{w \in \Delta^M, \ \a \in \Delta_M^{Nn}} \ \hat J(w,\a),
\end{equation*}
and similarly define $\hat q \coloneqq q_{\hat w, \hat \a}$. Whenever $\hat w$,  $\hat \a$, $q$, or $q_{\hat w, \hat \a}$ are referenced in the arbitrary group size setting, we will be referring to these estimators.

\subsubsection{Proof of Theorem \ref{oe}: Arbitrary Group Size}
We now state Theorem \ref{oe} for arbitrary group size.
\begin{manualtheorem}{\ref{oe}a} \label{oe_arb}
Given grouped observations of size $N$, let $\epsilon>0$ and $\delta = 2(Nn)^N\exp\left\{-\frac{\sigma^{2Nd}(n-N)\epsilon^2}{8C_k^{2N}}\right\}$. With probability at least $1-\delta$ the following holds: 
\begin{equation*}
    \left \lVert q - q_{\hat w, \hat \a} \right \rVert_2^2 \leq \  \inf_{w \in \Delta^M, \ \a \in \Delta_M^{Nn}} \left \lVert q - q_{w, \a} \right \rVert_2^2 + \epsilon.
\end{equation*}
\end{manualtheorem}

\begin{proof}
The proof proceeds as in the paired observation setting. In particular,
\begin{align*}
    P_{q} \{ &\sup_{\substack{w \in \Delta^M \\ \a_ \in \Delta_M^{Nn} } } | J(w,\a) - \hat J(w,\a) | > \frac{\epsilon}{2} \} \\
    &\leq P_{q} \Big\{ \sup_{\substack{w \in \Delta^M \\ \a_ \in \Delta_M^{Nn} } } \sum_{m,r,r'} \prod_{i \in [N] } w_m \a_{m,r_i,r_i'} \Big|h(r_1,r_1',\dots, r_N,r_N')   -\hat h(r_1,r_1',\dots, r_N,r_N')\Big| > \frac{\epsilon}{4}\Big\} \\
    &\leq P_{q} \Big\{ \max_{r_1,r_1',\dots, r_N,r_N'} \Big|h(r_1,r_1',\dots, r_N,r_N') -\hat h(r_1,r_1',\dots, r_N,r_N')\Big| > \frac{\epsilon}{4}\Big\} \\
    &\leq \sum_{r,r'}P_{q} \Big\{ \Big|h(r_1,r_1',\dots, r_N,r_N') -\hat h(r_1,r_1',\dots, r_N,r_N')\Big| > \frac{\epsilon}{4}\Big\}
\end{align*}

The first step above is due to the triangle inequality, and the penultimate step is due to simplex constraints on $w,\a$. Let $k_i(r_1,r_1',\dots,r_N,r_N') \coloneqq k_\sigma(x_{i,1}, x_{r_1,r_1'})k_\sigma(x_{i,2}, x_{r_2,r_2'})\cdots k_\sigma(x_{i,N}, x_{r_N,r_N'})$. Noting that $h(r_1,r_1',\dots, r_N,r_N') = \mathbb{E}_q\{k_\sigma(x_{i,1},x_{r_1,r_1'})\cdots k_\sigma(x_{i,N},x_{r_N,r_N'})\}$, we have
\begin{align*}
    P_{q} &\left \{ \Big|h(r_1,r_1',\dots, r_N,r_N') -\hat h(r_1,r_1',\dots, r_N,r_N')\Big|>\frac{\epsilon}{4} \right\}  \\ 
    &= P_{q}\Bigg\{ \Big| \sum_{i \in [n] \backslash L(r_1,\dots,r_N)}\mathbb{E}_{(x_{i,1},\dots,x_{i,N})\sim q}\{k_i(r_1,r_1',\dots,r_N,r_N')\}  \\
    &\qquad \qquad \qquad \qquad \qquad \qquad- k_i(r_1,r_1',\dots,r_N,r_N') \Big|>\frac{(n-N)\epsilon}{4} \Bigg\}
\end{align*}
The terms $k_i(r_1,r_1',\dots,r_N,r_N')$ are independent random variables due to use of the LNO estimator. By assumption, $0 \leq k_i(r_1,r_1',\dots,r_N,r_N')\leq C_k^N\sigma^{-Nd}$ so the $k_i$ are bounded for fixed $\sigma>0$. We apply Hoeffding's inequality
\begin{align*}
    P_{q} \left \{ \Big|h(r_1,r_1',\dots, r_N,r_N') -\hat h(r_1,r_1',\dots, r_N,r_N')\Big|>\frac{\epsilon}{4} \right \} &\leq 2\exp\{-\frac{2(n-N)^2\epsilon^2}{16(n-N)C_k^{2N}\sigma^{-2Nd}}\} \\
    &= 2\exp\{-\frac{\sigma^{2Nd}(n-N)\epsilon^2}{8C_k^{2N}}\}.
 \label{hoeffding}
\end{align*}
Substituting backward we obtain the desired upper bound
\begin{align*}
    P_{q} \{ \sup_{\substack{w \in \Delta^M \\ \a \in \Delta_M^{Nn}}} | J(w,\a) - \hat J(w,\a) | > \frac{\epsilon}{2} \} &\leq
    \sum_{r_1,r_1'} \dots \sum_{r_N,r_N'} 2\exp\{-\frac{\sigma^{2Nd}(n-N)\epsilon^2}{8C_k^{2N}}\} \\
    &= 2(Nn)^N\exp\{-\frac{\sigma^{2Nd}(n-N)\epsilon^2}{8C_k^{2N}}\}\\
\end{align*}
From here the proof is identical to the paired observation case, but with
\begin{equation*}
    \delta = 2(Nn)^N\exp\{-\frac{\sigma^{2Nd}(n-N)\epsilon^2}{8C_k^{2N}}\}.
\end{equation*}
\end{proof}

\section{Proof and General Form of Theorem \ref{l3}}\label{a:t2}
In this section we give the proof of Theorem \ref{l3} for groups of size two, before extending it to groups of arbitrary size. For readability, we first present some intermediate results to be used in the main proofs.

\subsection{Intermediate Results}
We first prove two supporting results.
\begin{lemma}\label{l2}
For any $1 \le p < \infty$, any $f,g \in L^p$, and any integer $a \ge 2$,
\begin{equation*}
    \|f^{\times a} - g^{\times a}\|_p \leq \|f\|^{a-1}_p\|f - g\|_p + \|g\|_p\|f^{\times (a-1)} - g^{\times (a-1)} \|_p,
\end{equation*}
where $f^{\times a}(y_1,y_2,\dots,y_a) := f(y_1)f(y_2)\cdots f(y_a)$.
\end{lemma}

\begin{proof}
Let $f,g \in L^p$, $1 \le p  < \infty$. Then
\begin{align*}
    \|f^{\times a} - g^{\times a} \|_p &=  \|f^{\times a} - f^{\times (a-1)}\times g + f^{\times (a-1)}\times g - g^{\times a} \|_p \\
    &\leq \|f^{\times a} - f^{\times (a-1)}\times g\|_p + \|f^{\times (a-1)}\times g - g^{\times a} \|_p \\
    &= \left( \int | f(x_1)\cdots f(x_{a-1}) (f(x_{a}) - g(x_{a})) |^pdx_1 \dots dx_{a} \right)^{\frac{1}{p}} \\ 
    & \quad + \left( \int |g(x_{a})(f(x_1)\cdots f(x_{a-1}) - g(x_1)\cdots g(x_{a-1}))|^pdx_1 \dots dx_{a} \right)^{\frac{1}{p}} \\
    &= \|f^{\times (a-1)}\|_p\|f - g\|_p + \|g\|_p\|f^{\times (a-1)} - g^{\times (a-1)} \|_p \\
    &= \|f\|^{a-1}_p\|f - g\|_p + \|g\|_p\|f^{\times (a-1)} - g^{\times (a-1)} \|_p .
\end{align*}
\end{proof}

We have the following corollary.

\begin{corollary} \label{convergenceImplication}
For any $1 \le p < \infty$, any $f,g \in L^p$, and any integer $a \ge 2$,
\begin{equation*}
    \|f^{\times a} - g^{\times a} \|_p \le \left( \sum_{b=1}^{a} \|f\|_p^{a-b} \|g\|_p^{b-1} \right) \|f - g\|_p.
\end{equation*}
\end{corollary}

\begin{proof}
The proof is by induction. Lemma \ref{l2} provides the base of the recursion for $a=2$. Now suppose the statement is true for $a \ge 2$. To prove the statement for $a+1$, we apply Lemma \ref{l2} again, together with the induction hypothesis, to get
\begin{align*}
    \|f^{\times (a+1)} - g^{\times (a+1)} \|_p &\leq\|f\|^a_p\|f - g\|_p + \|g\|_p\|f^{\times a} - g^{\times a} \|_p \\
    &\le \|f\|^a_p\|f - g\|_p + \|g\|_p \left( \sum_{b=1}^{a} \|f\|_p^{a-b} \|g\|_p^{b-1} \right) \|f - g\|_p \\
    &= \left( \sum_{b=1}^{a+1} \|f\|_p^{a+1-b} \|g\|_p^{b-1} \right) \|f - g\|_p.
\end{align*}
This completes the proof. 
\end{proof}

\subsection{Proof of Theorem \ref{l3}: Groups of Size Two}
We restate Theorem \ref{l3} for convenience. \begin{manualtheorem}{\ref{l3}}
If $\sigma \to 0$ and $\frac{n\sigma^{4d}}{\log n} \to \infty$ as $n \to \infty$, then $\left \lVert q - q_{\hat w, \hat \a}\right \rVert_1 \xrightarrow{a.s.} 0$.
\end{manualtheorem}

\begin{proof}
Lemma 3.1 of \cite{gyorfi90} states that if $\int \hat{q} = 1$ and $\| \hat{q} - q \|_2 \xrightarrow{a.s.} 0$, then $\| \hat{q} - q \|_1 \xrightarrow{a.s.} 0$. Since $\int \hat{q} =1$ in our case, our strategy is to show $\| \hat{q} - q \|_2 \xrightarrow{a.s.} 0$. To do this it suffices to show that 
\begin{equation}
    \label{eqn:esterror}
    \|q - \hat{q} \|_2^2 - \inf_{\substack{w \in \Delta^M \\ \a \in \Delta_M^{Nn}}} \|q - q_{w,\a} \|_2^2 \xrightarrow{a.s.} 0
\end{equation}
and 
\begin{equation}
    \label{eqn:approxerror}
    \inf_{\substack{w \in \Delta^M \\ \a \in \Delta_M^{Nn}}} \|q - q_{w,\a} \|_2 \xrightarrow{a.s.} 0.
\end{equation}
To show \eqref{eqn:esterror}, by the Borel-Cantelli lemma, it suffices to show that for all $\eps > 0$,
\[
\sum_{n=1}^\infty P_q \left( \|q - \hat{q} \|_2^2 - \inf_{\substack{w \in \Delta^M \\ \a \in \Delta_M^{2n}}} \|q - q_{w,\a} \|_2^2 \ge \eps \right) < \infty.
\]
Thus let $\eps > 0$. By Theorem \ref{oe_arb}, the probability in question is at most
\begin{align*}
    \delta &= 2(nN)^N\exp\left\{-\frac{(n-N)\sigma^{2Nd}\epsilon^2}{8C_k^{2N}}\right\} \\
    &= 8 \exp \left\{- 2 \log n \left( \frac{(n-2)\sigma^{4d}\eps^2}{16 C_k^4 \log n} - 1 \right) \right\}.
\end{align*}
By assumption on the growth of $n$ and $\sigma$, there exists $N_\eps$ such that for all $n \ge N_\eps$, 
\[
\frac{(n-2)\sigma^{4d}\eps^2}{16 C_k^4 \log n} \ge 2.
\]
For such $n$ we have
\[
\delta \le 8 \exp\{ -2 \log n \} = \frac8{n^2}
\]
which is summable.

To show \eqref{eqn:approxerror}, let $w^*$ be the true mixing weights from \eqref{q}. For $i=1,\dots,n$ let $e_i$ be the $m\in [M]$ such that $X_i = (x_{i,1},x_{i,2}) \stackrel{i.i.d.}{\sim} p^*_m$. Define 
\begin{align*}
    n_m &= \big| \{i: e_i = m\} \big|, &&m=1,\dots,M \\
    \a^*_{m,i,1} &= \a^*_{m,i,2} = \begin{cases}
    \frac{1}{2n_m}, & e_i=m \\
    0, & \text{otherwise}
    \end{cases}, &&m=1,\dots,M
\end{align*}
With this ``oracle" assignment of weights, $p(x;\a_m^*)$ is just the regular KDE for $p_m^*$. Therefore, we may apply known results for consistency of standard KDEs. In particular, we will apply Theorem 3.1 of \cite{gyorfi90} which implies 
\begin{equation}
 \| p(\cdot \,;\a_m^*) - p_m^* \|_2 \xrightarrow{a.s.} 0 \text{ as } n_m \to \infty \label{KDEconverge}
\end{equation}
provided $k \in L^2$ and $\sum_n \frac1{n^2 \sigma_n^d} < \infty$. Both of these conditions are satisfied by assumption in our setting. Furthermore, as $n \to \infty$ we have $\frac{n_m}{n}\to w_m^*$ almost surely, and therefore $n_m \to \infty$ almost surely.

Finally, we have
\begin{align*}
    \inf_{\substack{w \in \Delta^M \\ \a \in \Delta_M^{2n}}} \|q - q_{w,\a} \|_2
    &\le \|q - q_{w^*,\a^*} \|_2 \\  
    &= \left\| \sum_{m=1}^M w_m^* (p_m^* \times p_m^* - p(\cdot \, ; \a_m^*) \times p(\cdot; \a_m^*)) \right\|_2 \\
    &\le  \sum_{m=1}^M w_m^* \|p_m^* \times p_m^* - p(\cdot \,; \a_m^*) \times p(\cdot \,; \a_m^*)) \|_2 \\
    &\le  \sum_{m=1}^M w_m^* (\|p_m^*\|_2 + \| p(\cdot \,; \a_m^*) \|_2) \| p_m^* - p(\cdot \,; \a_m^*) \|_2 \\
    &\le  \sum_{m=1}^M w_m^* 3 \| p_m^* \|_2 \| p_m^* - p(\cdot \,; \a_m^*) \|_2 \\
    &\xrightarrow{a.s.} 0 \text{ as } n \to \infty,
\end{align*}
where the fourth step uses Lemma \ref{l2} and the fifth step holds for $n$ sufficiently large (a.s.). This completes the proof.
\end{proof}

\subsection{Proof of Theorem \ref{l3}: Arbitrary Group Size}
We consider the problem for arbitrary group size as described in Section \ref{arb_group_setup} of this document. The proof of Theorem \ref{l3} for arbitrary group size is similar to the proof for groups of size two. The main difference will be in use of Theorem \ref{oe_arb} rather than Theorem \ref{oe} to invoke the Borel-Cantelli lemma.

\begin{manualtheorem}{\ref{l3}a} \label{l3_arbitrary_group}
Given grouped observations of size $N \in \mathbb{Z}^+$, if $\sigma \to 0$ and $\frac{n\sigma^{2Nd}}{\log n} \to \infty$ as $n \to \infty$ then $\|q - \hat q \|_1 \xrightarrow{a.s.} 0 \text{ as } n \to \infty$.
\end{manualtheorem}

\begin{proof}
  We will appeal to Lemma 3.1 of \cite{gyorfi90} as we did for groups of size two. Namely, if $\int \hat{q} = 1$ and $\| \hat{q} - q \|_2 \xrightarrow{a.s.} 0$, then $\| \hat{q} - q \|_1 \xrightarrow{a.s.} 0$. Our strategy again is to show $\| \hat{q} - q \|_2 \xrightarrow{a.s.} 0$. To do this it suffices to show that 
\begin{equation}
    \label{eqn:esterror_arb}
    \|q - \hat{q} \|_2^2 - \inf_{\substack{w \in \Delta^M \\ \a \in \Delta_M^{Nn}}} \|q - q_{w,\a} \|_2^2 \xrightarrow{a.s.} 0
\end{equation}
and 
\begin{equation}
    \label{eqn:approxerror_arb}
    \inf_{\substack{w \in \Delta^M \\ \a \in \Delta_M^{Nn}}} \|q - q_{w,\a} \|_2 \xrightarrow{a.s.} 0.
\end{equation}
To show \eqref{eqn:esterror_arb}, by the Borel-Cantelli lemma, it suffices to show that for all $\eps > 0$,
\[
\sum_{n=1}^\infty P_q \left( \|q - \hat{q} \|_2^2 - \inf_{\substack{w \in \Delta^M \\ \a \in \Delta_M^{Nn}}} \|q - q_{w,\a} \|_2^2 \ge \eps \right) < \infty.
\]
Thus let $\eps > 0$. By Theorem \ref{oe_arb}, the probability in question is at most
\begin{align*}
    \delta &= 2(Nn)^N\exp\{-\frac{\sigma^{2Nd}(n-N)\epsilon^2}{8C_k^{2N}}\} \\
    &= 2N^N \exp \left\{N \log n + \left( \frac{(n-N)\sigma^{2Nd}\eps^2}{8 C_k^{2N} } \right) \right\} \\
    &= 2N^N \exp \left\{- N \log n \left( \frac{(n-N)\sigma^{2Nd}\eps^2}{8C_k^{2N}N \log n} - 1 \right) \right\}.
\end{align*}
By assumption on the growth of $n$ and $\sigma$, there exists $N_\eps$ such that for all $n \ge N_\eps$, 
\[
\frac{(n-N)\sigma^{2Nd}\eps^2}{8 C_k^{2N} N \log n} \ge 2.
\]
For such $n$ we have
\[
\delta \le 2N^N \exp\{ -N \log n \} = \frac{2N^N}{n^N}
\]
which is summable for $N > 1$.

To show \eqref{eqn:approxerror_arb}, let $w^*$ be the true mixing weights from \eqref{q_a}. For $i=1,\dots,n$ let $e_i$ be the $m\in [M]$ such that $X_i = (x_{i,1},x_{i,2},\dots,x_{i,N}) \stackrel{i.i.d.}{\sim} p^*_m$. Define 
\begin{align*}
    n_m &= \big| \{i: e_i = m\} \big|, &&m=1,\dots,M \\
    \a^*_{m,i,j} &= \begin{cases}
    \frac{1}{n_m N}, & e_i=m \\
    0, & \text{otherwise}
    \end{cases}, &&m=1,\dots,M, \ j=1,\dots,N
\end{align*}
We are again using an ``oracle" assignment of weights, so $p(x;\a_m^*)$ is just the regular KDE for $p_m^*$. Therefore, we may again apply Theorem 3.1 of \cite{gyorfi90} which implies
\begin{equation}
 \| p(\cdot \,;\a_m^*) - p_m^* \|_2 \xrightarrow{a.s.} 0 \text{ as } n_m \to \infty \label{KDEconverge_arb}
\end{equation}
provided $k \in L^2$ and $\sum_n \frac1{n^2 \sigma_n^d} < \infty$. Both of these conditions are satisfied by assumption in our setting. Furthermore, as $n \to \infty$ we have $\frac{n_m}{n}\to w_m^*$ almost surely, and therefore $n_m \to \infty$ almost surely.

Finally, we have
\begin{align*}
    \inf_{\substack{w \in \Delta^M \\ \a \in \Delta_M^{Nn}}} \|q - q_{w,\a} \|_2
    &\le \|q - q_{w^*,\a^*} \|_2  \\  
    &= \left\| \sum_{m=1}^M w_m^* (p_m^{*\times N}  - p(\cdot \, ; \a_m^*)^{\times N}) \right\|_2  \\
    &\le  \sum_{m=1}^M w_m^* \|p_m^{*\times N}  - p(\cdot \, ; \a_m^*)^{\times N} \|_2 \\
    &\le \sum_{m=1}^M w_m^* \left( \sum_{b=1}^{N} \|p_m^*\|_2^{N-b} \|p(\cdot \, ; \a_m^*)\|_2^{b-1} \right) \|p_m^*  - p(\cdot \, ; \a_m^*)\|_2 \\
    &\le \sum_{m=1}^M w_m^* \left( \sum_{b=1}^{N} 2^{b-1}\|p_m^*\|_2^{N-1} \right) \|p_m^*  - p(\cdot \, ; \a_m^*)\|_2 \\
    &\xrightarrow{a.s.} 0 \text{ as } n \to \infty,
\end{align*}
where the penultimate step uses \eqref{KDEconverge_arb} and Corollary \ref{convergenceImplication}, and the final step holds for $n$ sufficiently large (a.s.). This completes the proof.
\end{proof}

\section{Background on the Grouped Sample Setting and Proof of Theorem \ref{thm:compconv-maintext}}\label{a:t3}
Here we prove Theorem \ref{thm:compconv-maintext}. We will be proving a general and more technical version of this theorem, Theorem \ref{thm:compconv-supp}, from which Theorem \ref{thm:compconv-maintext} is a direct consequence. First we will introduce some background to the problem setting which was introduced in \cite{vandermeulen_operator_2019}. \textbf{This section uses its own notation which does not extend to other appendices or main text}.
\subsection{Identifiability in the Grouped Sample Setting}

We will be concerned with probability measures on a measurable space $\left(\Omega,\sF\right)$. Let $\delta$ be the Dirac measure. Let $\sD$ be the set of probability measures on $\left(\Omega,\sF\right)$.  We call a probability measure on $\sD$ of the form
\begin{equation*}
  \sP = \sum_{i=1}^m a_i \delta_{\mu_i}
\end{equation*}
a \emph{mixture of measures} \cite{vandermeulen_operator_2019}. For all mixtures of measures we will assume that $a_i>0$ for all $i$ and $\mu_i \neq \mu_j$ when $i\neq j$ so that $m$ is the number of distinct mixture components. The \emph{grouped sample} setting from \cite{vandermeulen_operator_2019} considers the situation where samples come in groups of size $n$ by first sampling a random measure component from a mixture of measures $\gamma \sim \sP$, which is then sampled iid $n$ times. So one has access to samples of the form $\bX = (X_1,\ldots,X_n)$ with $X_1,\ldots,X_n\simiid \gamma$. In this situation the identifiability of $\sP$ depends on whether the distribution of $\bX$ is uniquely determined by $\sP$ and the number of samples per group $n$. To this end \cite{vandermeulen_operator_2019} introduced the $V_n$ operator which maps a mixture of measures to the distribution of $\bX$:
\begin{equation*}
  V_n\left(\sum_{i=1}^m a_i \delta_{\mu_i}\right) = \sum_{i=1}^m a_i \mu_i^{\times n},
\end{equation*}
where $\mu^{\times n}$ denotes the product measure $n$ times. We note that $n=1$ corresponds to a typical mixture model where each mixture component is sampled once after being selected and there is no grouped sample structure.  For the grouped sample setting \cite{vandermeulen_operator_2019} introduces the following notion of identifiability.
\begin{definition}
A mixture of measures, $\sP = \sum_{i=1}^m a_i \delta_{\mu_i}$, is called \emph{$n$-identifiable} if there does not exist a different mixture of measures $\sQ = \sum_{j=1}^{m'} b_j \delta_{\nu_j}$, with $m'\le m$, such that $V_n\left( \sP \right) = V_n\left( \sQ \right)$.
\end{definition}
A \emph{completely} rigorous mathematical treatment of the previous notions is a bit involved and can be found in \cite{vandermeulen_operator_2019}. In \cite{vandermeulen_operator_2019} it is shown that if the mixture components are jointly irreducible then a mixture of measures is 2-identifiable, if they are linearly independent then they are 3-identifiable, and that any mixture of measures with $m$ components is $(2m-1)$-identifiable. 

\subsection{Notation}
Before we state and prove the main theorem of this section we need to first introduce some notation.

Let $S_m$ be the symmetric group over $m$ symbols. Abusing notation slightly we will let the elements of $S_m$ be a group action on $[m]$ as well as $\rn^m$. On $\rn^m$ it is defined as the following 
\begin{align}
  \sigma\left(\left[x_1,\ldots,x_m \right]^T \right) = \left[x_{\sigma(1)},\ldots,x_{\sigma(m)} \right]^T.
\end{align}
We also let $S_m$ be an operator where $S_m \cdot x $ is the orbit of $x$, i.e.
\begin{align}
  S_m\cdot x \triangleq \left\{\sigma(x): \sigma \in S_m \right\}.
\end{align}

Recall that for a pair of Hilbert spaces $H,H'$ the direct sum $H\oplus H'$ is a Hilbert space with elements of the form $x\oplus x'$ and inner product defined as $\left<x\oplus x',y\oplus y'\right> = \left<x,y\right> + \left<x',y'\right>$. For a pair of Banach spaces $B,B'$ we define the direct sum via the norm $\left\|b\oplus b'\right\|_{B\oplus B'} \triangleq \left\|b\right\|_B + \left\|b'\right\|_{B'}$ which is itself a Banach space (\cite{banach32} p. 183).

For a pair of Hilbert spaces $H,H'$ let $H \otimes H'$ be the tensor product of these two spaces and $h\otimes h'$ be the tensor product of vectors $h\in H$ and $h'\in H'$. For a vector in a Hilbert space $h$ let $h^{\otimes n}$ denote the tensor power, i.e. $\underbrace{h\otimes \cdots \otimes h}_{n\text{ times}}$.

In the following the space of finite signed measures is equipped with the total variation topology and unadorned norms refer to the total variation norm on finite signed measures, which forms a Banach space. Norms for various Lebesgue spaces will have the associated subscript. Finally we note that for two Hilbert spaces of square-integrable functions over $\sigma$-finite measure spaces  $L^2\left(\Omega, \sF,\mu \right)$ and $L^2\left(\Omega', \sF',\mu' \right)$ we have that $L^2\left(\Omega, \sF,\mu \right) \otimes L^2\left(\Omega', \sF',\mu' \right) \cong L^2\left(\Omega \times \Omega', \sF \times \sF',\mu\times \mu' \right)$ via an isomorphism $f\otimes f' \mapsto f\times f'$ (\cite{kadison83} Example 2.6.11) and we will use a 2 subscript for both norms.
\subsection{Full Theorem Statement and Proof}
The following is the full general version of Theorem \ref{thm:compconv-maintext} and the main result of this section. 
\begin{theorem}
\label{thm:compconv-supp}
  Let $(\Omega, \sF)$ be a measurable space, $\sP = \sum_{j=1}^m a_j \delta_{\mu_j}$ a mixture of measures on that space which is $n$-identifiable, and $\sP_i = \sum_{j=1}^{m'_i} b_{i,j} \delta_{\nu_{i,j}}$  a sequence of mixtures of measures with $m_i'\le m$ for all $i$, such that $V_n(\sP_i)\to V_n(\sP)$. Then $m_i'\to m$ and there exists a sequence of permutations $\sigma_i$ such that $\sigma_i\left(b_i\right) \to a$ and $\nu_{i,\sigma_i(j)} \to \mu_j$ for all $j$.
\end{theorem}
Essentially this says that as one finds grouped sample distributions $V_n(\sQ_i)$ which approach the true grouped sample distribution $V_n(\sP)$ the mixture of measures $\sQ_i$  will automatically recover the true mixing weights and components from $\sP$ so long as $\sP$ is $n$-identifiable. In other words, one simply needs to fit the grouped distribution $V_n(\sP)$ well to get a good estimate of the mixture components. Theorem \ref{thm:compconv-maintext} from the main text is a direct consequence of Theorem \ref{thm:compconv-supp}.
\begin{corollary}[Theorem \ref{thm:compconv-maintext}] 
 Let $\sum_{m=1}^M w_m p_m$ be an $N$-identifiable mixture model, and $\sum_{m=1}^M \hat{w}_{m,j} \hat{p}_{m,j}$ be a sequence of mixture models such that $\left\|\sum_{m=1}^M \hat{w}_{m,j} \hat{p}_{m,j}^{\times N} - \sum_{m=1}^M w_m p_m^{\times N} \right\|_1 \to 0$. Then there is a sequence of permutations $\sigma_j$ so that $\hat{w}_{\sigma_j(m),j} \to w_m$ and $\left\|\hat{p}_{\sigma_j(m),j} - p_m\right\|_1 \to 0$ for all $m$.
\end{corollary}

We introduce some preliminary results before proving Theorem \ref{thm:compconv-supp}. The following lemma will be needed for our proof.
\begin{lemma} \label{lem:identineq}
  Let $\sP$ and $\sQ$ be mixtures of measures, then 
  $ \left\|V_{n'}(\sP) - V_{n'}(\sQ)\right\|_{} \le \left\|V_{n}(\sP) - V_{n}(\sQ)\right\|$
  for all $n'\le n$.
\end{lemma}
\begin{proof}[Proof of Lemma \ref{lem:identineq}]
  From \cite{folland99} (Section 3.1 Exercise 7a) we have the following
  \begin{align*}
    &\left\|V_{n}(\sP) - V_{n}(\sQ)\right\| \\
    &= \sup\Bigg\{ \sum_{i=1}^k \left|(V_{n}(\sP) - V_{n}(\sQ))(E_i)\right| : \\
    & \qquad \qquad \qquad k\in \nn, E_1,\ldots,E_k \in \sF^{\times n} \text{ are disjoint, and } \bigcup_{i=1}^k E_i  = \Omega^{\times n} \Bigg\}\\
    &\ge \sup\Bigg\{ \sum_{i=1}^k \left|(V_{n}(\sP) - V_{n}(\sQ))(E_i\times \Omega^{\times n- n'})\right|: \\
    & \qquad \qquad \qquad k\in \nn, E_1,\ldots,E_k \in\sF^{\times n'} \text{ are disjoint, and } \bigcup_{i=1}^k E_i  = \Omega^{\times n'} \Bigg\}\\
    &= \sup\Bigg\{ \sum_{i=1}^k \left|(V_{n'}(\sP) - V_{n'}(\sQ))(E_i)\right|: \\
    & \qquad \qquad \qquad k\in \nn, E_1,\ldots,E_k \in\sF^{\times n'} \text{ are disjoint, and } \bigcup_{i=1}^k E_i  = \Omega^{\times n'} \Bigg\}\\
    & =\left\|V_{n'}(\sP) - V_{n'}(\sQ)\right\|. 
  \end{align*}
\end{proof}
The following lemma is the main workhorse in the proof of Theorem \ref{thm:compconv-supp}.
\begin{lemma}\label{lem:compconv}
  Let $(\Omega, \sF)$ be a measurable space, $\sP = \sum_{i=1}^m a_i \delta_{\mu_i}$ a mixture of measures on that space, $n \in \nn$, and $\sP_i = \sum_{j=1}^{m'} b_{j} \delta_{\nu_{i,j}}$ a sequence of mixtures of measures ($m'$ is fixed) with such that $V_n(\sP_i)\to V_n(\sP)$ ($b$ does not depend on $i$). Then there exists a subsequence $i_k$ and a collection of probability measures $\nu_{1},\ldots,\nu_{m'}$ such that $\nu_{i_k,j} \to \nu_j$ for all $j$ and $V_n\left(\sP\right) = V_n \left( \sum_{j=1}^{m'}b_j\delta_{\nu_j}\right)$.
\end{lemma}
\begin{proof}[Proof of Lemma \ref{lem:compconv}]
  We will use bold symbols to represent elements that depend on $i$, e.g.\ $\bnu_j = \nu_{i,j}$. Let $\barmu = \sum_{k=1}^m a_k \mu_k$. By the Lebesgue-Radon-Nikodym Theorem (\cite{folland99} Theorem 3.8) there exists series of measures $\blambda_1,\ldots,\blambda_{m'}$ and $\brho_1,\ldots,\brho_{m'}$ such that $\bnu_k = \blambda_k + \brho_k$ with $\blambda_k \perp \barmu$ and $\brho_k \ll \barmu$ for all $k\in\left[m'\right]$.

  For some fixed $\ell$ let $\bA_\ell$ be the sequence of measurable sets such that $\blambda_\ell\left(\cdot \cap \bA_\ell\right) = \blambda_\ell$ and $\barmu\left(\bA_\ell\right) = 0$, this is possible since $\blambda_\ell \perp \barmu$. From Lemma \ref{lem:identineq} we have that
  \begin{align}
    \left\|\sum_{k=1}^m a_k \mu_k - \sum_{j=1}^{m'} b_j \bnu_j  \right\| \to 0
    &\Rightarrow \left|\sum_{k=1}^m a_k \mu_k(\bA_\ell) - \sum_{j=1}^{m'} b_j \bnu_j(\bA_\ell)  \right| \to 0\\
    &\Rightarrow \left|b_\ell\brho_\ell(\bA_\ell) + b_\ell\blambda_\ell(\bA_\ell) + \sum_{j\in [m']\setminus\left\{\ell\right\}} b_j \bnu_j(\bA_\ell)  \right| \to 0\\
    &\Rightarrow \left| b_\ell\blambda_\ell(\bA_\ell) + \sum_{j\in [m']\setminus\left\{\ell\right\}} b_j \bnu_j(\bA_\ell)  \right| \to 0.
  \end{align} 
  Because all of the summands inside the absolute value on the last line are positive we have that $\left\|\blambda_\ell\right\| \to 0$ and thus $\left\|\brho_\ell\right\| \to 1$. Eventually in our sequence we must have that $\left\|\brho_\ell\right\| >0$, so eventualy in our subsequence we can define $\bnu_\ell' = \brho_\ell/\left\|\brho_\ell\right\|$ which is now a sequence of probability measures which are absolutely continuous with respect to $\barmu$ and $\left\|\bnu_\ell' - \bnu_\ell\right\| \to 0$. 

  From this we have that there exists sequences of probability measures $\bnu'_1,\ldots,\bnu'_{m'}$ such that $\left\|\bnu_k - \bnu'_k\right\|\to 0$ and $\bnu'_k \ll \barmu$ for all $k\in [m']$. Lemma 3.3.7 in \cite{reiss89} states that, for probablity measures over the same domain $\xi_1,\ldots \xi_d,\gamma_1,\ldots,\gamma_d$ that $\left\|\prod_{j=1}^d \xi_j - \prod_{k=1}^d \gamma_k\right\|\le \sum_{k=1}^d \left\|\xi_k - \gamma_k \right\|$. It follows therefore that $\left\|{\bnu'}_k^{\times n} - {\bnu}_k^{\times n}\right\| \to 0$  for all $k$ and 
  \begin{align}
    \left\|\sum_{k=1}^m a_k \mu_k^{\times n} - \sum_{j=1}^{m'} b_j {\bnu'}_j^{\times n}  \right\| \to 0.\label{eqn:nuprime}
  \end{align}

  For some fixed $\ell$ let $\bq_\ell'$  be the Radon-Nikodym derivative of $\bnu'_\ell$ with respect to $\barmu$.  Let $\bB_\ell= {\bq'}_\ell^{-1}\left(\left[2/b_\ell,\infty\right)\right)$. We have the following
  \begin{align}
    \sum_{k=1}^{m'} b_k \bnu'_k(\bB_\ell)
    &\ge b_\ell \bnu'_\ell(\bB_\ell) \\
    &\ge b_\ell \int_{\bB_1} 2/b_\ell d\barmu \\
    &\ge  2\barmu(\bB_\ell).
  \end{align}
  From Lemma \ref{lem:identineq} applied to (\ref{eqn:nuprime}) we have that $\left|\sum_{k=1}^{m'} b_k \bnu'_k(\bB_\ell) - \barmu(\bB_\ell)\right| \to 0$, and it follows that $\barmu\left(\bB_\ell\right) \to 0$ because $\left|\sum_{k=1}^{m'} b_k \bnu'_k(\bB_\ell) - \barmu(\bB_\ell)\right|\ge \barmu(\bB_\ell)$. Now we have that $\sum_{k=1}^{m'} b_k \bnu'_k(\bB_\ell) \to 0$ and thus $\bnu'_\ell(\bB_\ell) \to 0$.

  Because $\bnu'_\ell\left( {\bq'}_\ell^{-1}\left(\left[2/b_\ell,\infty\right)\right) \right) \to 0$ and therefore $\bnu'_\ell\left(\bB_\ell^C\right) \to 1$, for sufficiently large $i$ we can now define a sequence of probability measures $\bnu_\ell''$ via $\bnu_\ell''(A) = \bnu_\ell'(A\cap\bB_\ell^C)/\bnu'_\ell(\bB_\ell^C).$ We have that 
  \begin{align}
    \left\| \bnu'_\ell - \bnu''_\ell\right\|
    &=\left\|\left( \bnu'_\ell \left(\bB_\ell\cap \cdot\right) + \bnu'_\ell \left(\bB_\ell^C\cap \cdot\right)\right)- \left( \bnu''_\ell \left(\bB_\ell\cap \cdot\right) + \bnu''_\ell \left(\bB_\ell^C\cap \cdot\right)\right)\right\|\\
    &\le\left\| \bnu'_\ell \left(\bB_\ell\cap \cdot\right)- \bnu''_\ell \left(\bB_\ell\cap \cdot\right)\right\| +\left\| \bnu'_\ell \left(\bB_\ell^C\cap \cdot\right) - \bnu''_\ell \left(\bB_\ell^C\cap \cdot\right) \right\|\\
    &=\bnu'_\ell\left(\bB_\ell\right)+ \left\| \bnu'_\ell \left(\bB_\ell^C\cap \cdot\right)- \bnu'_\ell \left(\bB_\ell^C\cap \cdot\right)/\bnu'_\ell\left(\bB_\ell^C\right)\right\|\\
    &=\bnu'_\ell\left(\bB_\ell\right)+\left|1 - 1/\bnu'_\ell\left(\bB_\ell^C\right) \right|\left\| \bnu'_\ell \left(\bB_\ell\cap\cdot \right)\right\|
  \end{align}
  which goes to zero, so $\left\|\bnu_\ell''-\bnu_\ell\right\|\to 0.$
  Note that $\bnu_\ell''$ is a sequence of probability measures with Radon-Nikodym derivatives $\bq''_\ell \triangleq\bq'_\ell \b1_{\bB_\ell^C}/ \bnu'_\ell\left(\bB_\ell^C\right)$ ($\b1$ is the indicator function) and thus
  \begin{align*}
    \sup_x \bq''_\ell(x) = \sup_x\bq'_\ell (x) \b1_{\bB_\ell^C} (x)/ \bnu_\ell\left(\bB_\ell^C\right) \le 2/(b_\ell\bnu_\ell\left(\bB_\ell^C\right))
  \end{align*}
  and since $\bnu_\ell\left(\bB_\ell^C\right) \to 1$ eventually $\left\| \bq''_\ell\right\|_{\infty} \le 3/b_\ell$.
  From this we have that $\bq''_\ell \in L^1\left(\Omega,\sF, \barmu\right)\cap L^\infty\left(\Omega,\sF, \barmu\right)$ and $\left\|\bq_\ell''\right\|_\infty$ is a bounded sequence. From H\"olders's Inequality we have that 
  $$ = \left\|\bq''_\ell\right\|_2^2  = \left\|\bq''_\ell \bq''_\ell\right\|_1^2 \le \left\|\bq''_\ell\right\|_1 \left\|\bq''_\ell\right\|_\infty = \left\|\bq''_\ell\right\|_\infty $$
  so $\bq''_\ell$ is a bounded sequence in $L^2\left(\Omega,\sF, \barmu\right)$.

  We now define $\bnu''_1,\ldots,\bnu''_{m'}$ $\bq''_1,\ldots,\bq''_{m'}$ similarly. There exists $\beta$ such that $\left\|\bq''_j\right\|_\infty \le \beta$ and $\left\|\bq''_j\right\|_2 \le \beta$ along the whole series and for all $j$. Let $p_1,\ldots, p_m$ be the radon Nikodym derivatives for $\mu_1,\ldots,\mu_m$ with respect to $\barmu$, again these are  in $L^1\left(\Omega,\sF, \barmu\right)\cap L^2\left(\Omega,\sF, \barmu\right)\cap L^\infty\left(\Omega,\sF, \barmu\right)$. To see this note that $p_i \le 1/a_i$ otherwise we have that
  \begin{align*}
    \mu_i\left( p_i^{-1}\left(\left(1/a_i,\infty\right)\right) \right)
    &= \int_{p_i^{-1}\left(\left(1/a_i,\infty\right)\right)} p_i d\barmu\\
    &> \int_{p_i^{-1}\left(\left(1/a_i,\infty\right)\right)} 1/a_i d\barmu\\
    &> \sum_j \int_{p_i^{-1}\left(\left(1/a_i,\infty\right)\right)} 1/a_i  a_jd\mu_j\\
    &\ge \mu_i\left( p_i^{-1}\left(\left(1/a_i,\infty\right)\right) \right)
  \end{align*}
  a contradiction.
  Now we have
  \begin{align}
    \left\|\sum_{k=1}^m a_k p_k^{\times n} - \sum_{j=1}^{m'} b_j {\bq_j''}^{\times n}\right\|_1\to 0. 
  \end{align}
  and Lemma \ref{lem:identineq} implies
  \begin{align}
    \left\|\sum_{k=1}^m a_k p_k^{\times 2} - \sum_{j=1}^{m'} b_j {\bq_j''}^{\times 2}\right\|_1\to 0. \label{eqn:applyholders}
  \end{align}
  From H\"older's Inequality ($\left\|f\right\|_2^2 \le \left\|f\right\|_1 \left\|f\right\|_\infty$) we have that 
  \begin{align}
    \left\|\sum_{k=1}^m a_k p_k^{\times 2} - \sum_{j=1}^{m'} b_j {\bq_j''}^{\times 2}\right\|_2\to 0
  \end{align}
  and 
  \begin{align}
    \left\|\sum_{k=1}^m a_k p_k^{\otimes 2} - \sum_{j=1}^{m'} b_j {\bq_j''}^{\otimes 2}\right\|_2\to 0.
  \end{align}
  Let $S = \spn\left(\left\{p_1,\ldots,p_m\right\}\right)$ and $\ell\in \left[m'\right]$ be arbitrary. We have that $\bq''_\ell = \proj_S(\bq''_\ell) + \proj_{S^{\perp}}(\bq''_\ell)$, noting that the summands in the decomposition are both $L^2$ bounded sequences. So now we have that

  \begin{align}
    &\left< \sum_{k=1}^{m'} b_k {\bq_k''}^{\otimes 2}- \sum_{j=1}^m a_j p_j^{\otimes 2} ,\proj_{S^{\perp}}(\bq''_\ell)^{\otimes 2}\right>\to 0\\
    \Rightarrow&\left<\sum_{j=1}^{m'} b_j {\bq_j''}^{\otimes 2},\proj_{S^{\perp}}(\bq''_\ell)^{\otimes 2}\right>\to 0\\
    \Rightarrow& b_\ell\left<\proj_{S^{\perp}}(\bq''_\ell)^{\otimes 2}, \proj_{S^{\perp}}(\bq''_\ell)^{\otimes 2}\right>+\sum_{j\in[m']\setminus\left\{\ell\right\}} b_j\left< \bq_j'',\proj_{S^{\perp}}(\bq''_j)\right>^2\to 0\\
    \Rightarrow& b_\ell\left\|\proj_{S^{\perp}}({\bq''}_\ell)\right\|_2^4 \to 0.
  \end{align}
  From this we have that $\left\|\proj_S(\bq''_k) - \bq''_k\right\|_2\to 0$ for all $k$. Since $\bigoplus_{j=1}^{m'} \proj_S(\bq''_j)$ is a $L^2$ bounded sequence on a finite dimensional space by the Bolzano-Weierstrass theorem it has a convergent subsequence which converges to $\bigoplus_{j=1}^{m'} q''_j$ so  $ \bq_j''\to q''_j$ in $L^2$.
  From H\"older's Inequality we have that, along this subsequence
  \begin{align}
    \left\|\bq''_k -  q''_k\right\|_1
    \le \left\|\bq''_k -  q''_k\right\|_2 \left\|1\right\|_2
    \le \left\|\bq''_k -  q''_k\right\|_2 \sqrt{\int 1^2 d\barmu}
    =\left\|\bq''_k -  q''_k\right\|_2 \to 0
  \end{align}
  so $q_k''$ is a probability density for all $k$, since they must be nonnegative to converge and integrate to one. Now we have that
  \begin{align}
    \sum_{j=1}^m a_j p_j^{\times n} = \sum_{k=1}^{m'} b_k {q''_k}^{\times n}.
  \end{align}
  And defining $\nu_k$ as the probability measure associated with $q''_k$ we have that there exists a subsequence such that $\left\|\bnu_k - \nu_k\right\|\to 0 $ for all $k$ and
  \begin{align}
    \sum_{j=1}^m a_j \mu_j^{\times n} = \sum_{k=1}^{m'} b_k {\nu_k}^{\times n}.
  \end{align}
\end{proof}

We can now prove Theorem \ref{thm:compconv-supp}.
\begin{proof}[Proof of Theorem \ref{thm:compconv-supp}]
  To help lighten notation we will simply bold some elements which depend on the sequence $\sP_i$. Let $\sP_i = \sum_{j=1}^{m'_i} \bb_j \delta_{\bnu_j}$ be a sequence of mixtures of measures ($\bb_j$, $\bnu_j$ are functions of $i$) such that $V_n(\sP_i)\to V_n(\sP)$. 

  We define $\tbb$ a sequence in $\Delta^{m}$ so that $\tbb_j = \bb_j$ for $j\le m_i'$ and $\tbb_k = 0$ for $k>m_i'$. Consider the case where there exists no sequence of permutations such that $\bsigma(\tbb) \to a$. From this it would follow that there exists a subsequence on $i$ and $\varepsilon>0$ such that $\left\|\tbb-\sigma(a)\right\| > \varepsilon$ for all $\sigma\in S_m$. The space 
  \begin{equation}
    \Delta^{m} \cap \left(\bigcap_{\sigma \in S_m} \operatorname{ball}\left(\sigma(a),\varepsilon\right)^C\right)
  \end{equation}
  is compact (the ball is open) so there exists a sub-subsequence of $i$ where $\tbb$ converges to a point $b \not\in S_m\cdot a$. Let $I \subset [m]$ be the indices of $b$ which are nonzero and $m' = \max (I)$. For sufficiently large $i$ along our sub-subsequence we have that $m'_i\ge m'$ and furthermore
  \begin{align}
    \left\|\sum_{j=1}^{m'_i} \bb_j \bnu_j^{\times n} - \sum_{k\in I} b_k \bnu_k^{\times n}  \right\| 
    &\le \left\|\sum_{j\in I} \bb_j \bnu_j^{\times n} - \sum_{k\in I} b_k \bnu_k^{\times n}  \right\| +\left\|\sum_{j\in I^C} \bb_j \bnu_j^{\times n} \right\| \\
    &\le \left\|\sum_{j\in I} (\bb_j - b_j) \bnu_j^{\times n} \right\| +\left|\sum_{j\in I^C} \bb_j  \right| \\
    &\le \sum_{j\in I}\left| \bb_j - b_j  \right| +\left|\sum_{j\in I^C} \bb_j  \right|\to 0
  \end{align}
  and therefore
  \begin{align}
    \left\|\sum_{k=1}^m a_k \mu_k^{\times n} - \sum_{j\in I} b_j \bnu_j^{\times n}  \right\| \to 0.
  \end{align}
  From Lemma \ref{lem:compconv} we have that there exists a subsequence of this sub-subsequence such that for $k \in I$ there exists probability measures $\nu_k$ with $\left\|\nu_k - \bnu_k\right\|\to 0$ and
  \begin{align}
    \sum_{k=1}^m a_k \mu_k^{\times n}= \sum_{j\in I} b_j \nu_j^{\times n}.
  \end{align}
  If $|I|<m$ or $\nu_j = \nu_k$ for any $k\neq j$ and $j,k\in I$ then we have clearly violated identifiability since we can construct a mixture of measures $\sP'$ with fewer components than $\sP$ and $V_n\left(\sP'\right)= V_n\left(\sP\right)$. If $|I| = m$ (i.e. $I = [m]$) and $\nu_j$ are all distinct we have also arrived at a contradiction since letting $\sP' = \sum_{j=1}^m b_j \delta_{\nu_j} \neq \sP$ because there exists no $\sigma$ such that $\sigma(b) =a$ and $V_n(\sP') = V_n(\sP)$, contradicting identifiability.

  So we have that for sufficiently large $i$ that $m_i' = m$ and there exists at least one sequence $\bsigma$ such that $\bsigma(\bb) \to a$. So let $\left\| \sum_{k=1}^m a_k \mu_k^{\times n} - \sum_{j=1}^m \bb_j \bnu_j^{\times n}\right\| \to 0$. From what we have just shown, we can permute the indices and, without loss of generality, we can assume that $\bb \to a$. So now we have that $\left\| \sum_{i=1}^m a_i \mu_i^{\times n} - \sum_{j=1}^m a_j \bnu_j^{\times n}\right\| \to 0$.

  Let $\tS_m\subset S_m$ be the subgroup of permutations such that $\sigma(a) = a$ for $\sigma \in \tS_m$ (also known as the \emph{stabilizer} of $a$). Note that if $a_1,\ldots,a_m$ are distinct then $\tS_m$ only contains the identity. We proceed by contradiction: suppose there exists no sequence of permutations $\bsigma \in \nn^{\tS_m}$ such that $\bnu_{\bsigma(k)}\to \mu_k$ for all $k$.
  From this it follows that there exists a subsequence and a $\varepsilon >0$, such that $\bigoplus_{k=1}^m \bnu_k$ does not lie in  $\bigcap_{\sigma \in \tS_m}\left(\operatorname{ball}\left(\bigoplus_{k=1}^m \mu_{\sigma(k)}\right),\varepsilon\right)^C$. From Lemma \ref{lem:compconv} there exists probability measures, $\nu_1,\ldots,\nu_m$ such that for some subsequence $\left\|\bnu_k - \nu_k\right\|\to 0 $ for all $k$ and
  $$\sum_{j=1}^m a_j \mu_j^{\times n} = \sum_{k=1}^m a_k \nu_k^{\times n}.$$
   Because $\bigcap_{\sigma \in \tS_m}\left(\operatorname{ball}\left(\bigoplus_{k=1}^m \mu_{\sigma\left(k\right)}\right),\varepsilon\right)^C$ is closed we have $\bigoplus_{j=1}^m\nu_j \in \bigcap_{\sigma \in \tS_m}\left(\operatorname{ball}\left(\bigoplus_{k=1}^m \mu_{\sigma\left(k\right)}\right),\varepsilon\right)^C$ and there exists no $\sigma \in \tS_m$ such that  $\nu_{\sigma(k)} = \mu_k$ for all $k$ so. Setting $\sP' = \sum_{k=1}^m a_k \delta_{\nu_k}$ we have that $\sP' \neq \sP$ but $V_n\left(\sP'\right) = V_n\left(\sP\right)$, a contradiction. \sloppy

\end{proof}

\section{General Version of Corollary \ref{cor:main} } \label{a:c1}

Here we present the general version of Corollary \ref{cor:main} which guarantees recovery of the true mixture components using our estimator for any mixture model, provided there are a sufficient number of samples per group. For a mixture model $p = \sum_{m=1}^M w_m^* p_m^*$, using the estimator $\hat{q}$ from Section \ref{arb_group_setup} to estimate (\ref{q_a}):
 \begin{equation*}
q(y_1, y_2, \dots, y_N) = \sum_{m=1}^M w^*_m p^*_m(y_1)p^*_m(y_2) \dots p^*_m(y_N), \quad y_1,y_2,\dots,y_N \in \R^d. 
\end{equation*} 
combining Theorem \ref{l3_arbitrary_group} and Theorem \ref{thm:compconv-maintext} gives the following result. 
\begin{corollary}
If $\sigma \to 0$ and $\frac{n\sigma^{2Nd}}{\log n} \to \infty$ as $n \to \infty$, and 
$p$ is $N$-identifiable (e.g. $N=2M-1$), then $\hat w_m \stackrel{a.s.}{\to} w_m^*$ and $\|p(\cdot; \hat \alpha_m) - p_m^*\|_1 \stackrel{a.s.}{\to} 0$, up to a permutation.
\end{corollary}

\end{appendices}

\end{document}